%% file: main.tex
\documentclass{article}

\usepackage[utf8]{inputenc} 
\usepackage[T1]{fontenc}    
\usepackage{hyperref}       
\usepackage{url}            
\usepackage{booktabs}       
\usepackage{amsfonts}       
\usepackage{nicefrac}       
\usepackage{microtype}      
\usepackage{xcolor}         

\usepackage{hyperref}       
\usepackage{url}            
\usepackage{booktabs}       
\usepackage{amsfonts}       

\usepackage{graphicx}
\usepackage{subcaption}

\usepackage{tcolorbox}
\definecolor{darkred}{RGB}{150,0,0}
\definecolor{darkgreen}{RGB}{0,150,0}
\definecolor{darkblue}{RGB}{0,0,150}
\hypersetup{colorlinks=true, linkcolor=red, citecolor=darkblue, urlcolor=darkblue}

\usepackage{tikz}
\usepackage{amssymb}
\usepackage{amsmath}
\usepackage{amsmath,amsfonts,amssymb,amsthm}
\usepackage{mathtools}
\usepackage{commath}
\usepackage{flexisym}
\usepackage[utf8]{inputenc}
\usepackage{csquotes}
\usepackage[english]{babel}

\usepackage{color}
\usepackage{bbm}

\usepackage{amsmath}
\DeclareMathOperator*{\argmax}{arg\,max}
\DeclareMathOperator*{\argmin}{arg\,min}

\newtheorem{lemma}{Lemma}
\newtheorem{example}{Example}

\newtheorem{theorem}{Theorem}
\newtheorem{myassum}{Assumption}
\newtheorem{remark}{Remark}
\theoremstyle{definition}

\usepackage[shortlabels,inline]{enumitem}
\usepackage{lipsum}

\usepackage{enumitem}
\usepackage{xcolor}
\usepackage[linesnumbered,ruled,vlined]{algorithm2e}

\SetCommentSty{mycommfont}

\SetKwInput{KwInput}{Input}                
\SetKwInput{KwOutput}{Output} 
\usepackage[margin=1in]{geometry}

\newcommand{\tocite}[1]{\textcolor{green}{[CITE]}}

\input{commands.tex}

\usepackage{authblk}
\author[1]{Sanae Amani}
\author[2]{Lin F. Yang}
\author[3]{Ching-An Cheng}
\affil[1,2]{University of California, Los Angeles}
\affil[3]{Microsoft Research}
{
    \makeatletter
    \renewcommand\AB@affilsepx{, \protect\Affilfont}
    \makeatother
    \affil[1]{samani@ucla.edu}
\affil[2]{linyang@ee.ucla.edu}
\affil[3]{chinganc@microsoft.com}
}

\title{Provably Efficient Lifelong Reinforcement Learning\\with Linear Function Approximation}

\begin{document}

\sloppy
\date{}
\maketitle

\begin{abstract}
We study lifelong reinforcement learning (RL) in a regret minimization setting of linear contextual Markov decision process (MDP), where the agent needs to learn a multi-task policy while solving a streaming sequence of tasks.
We propose an algorithm, called UCB Lifelong Value Distillation (\UCBlvd), that provably achieves sublinear regret for any sequence of tasks, which may be adaptively chosen based on the agent's past behaviors. Remarkably, our algorithm uses only sublinear number of planning calls, which means that the agent eventually learns a policy that is near optimal for multiple tasks (seen or unseen) without the need of deliberate planning. A key to this property is a new structural assumption that enables computation sharing across tasks during exploration.
Specifically, for $K$ task episodes of horizon $H$, our algorithm has a regret bound $\Otilde(\sqrt{(d^3+\dpr d)H^4K})$  based on $\Oc(dH\log(K))$ number of planning calls, where $d$ and $\dpr$ are the feature dimensions of the dynamics and rewards, respectively. 
This theoretical guarantee implies that our algorithm can enable a lifelong learning agent to accumulate experiences and learn to rapidly solve new tasks.
%
%

\end{abstract}

\section{Introduction}

Recently there has been a surging interest in designing \emph{lifelong learning} agents that can continuously learn to solve multiple sequential decision making problems in its lifetime~\cite{thrun1995lifelong,silver2013lifelong,xie2021lifelong}.
This scenario is in particular motivated by building multi-purpose embodied intelligence~\cite{roy2021machine}, such as robots working in a weakly structured environment. 
Typically, curating all tasks beforehand for such problems is nearly infeasible, and the problems the agent is tasked with may be adaptively selected based on the agent's past behaviors.
Consider household robot as an example. Because each household is unique, it is difficult to anticipate upfront all scenarios the robot would encounter. Moreover, the tasks the robot faces are not independent and identically distributed (i.i.d.). 
Instead, what the robot has done before can affect the next task and its starting state; e.g., if the robot fails to bring a glass of water and breaks it, then the user is likely to command the robot to clean up the mess. Therefore, it is critical that the agent can continuously improve and generalize learned abilities to different tasks, regardless of their order.


In this work, we study lifelong reinforcement learning (RL) theoretically in a regret minimization setting~\cite{thrun1995lifelong,ammar2015safe}, where the agent needs to solve a sequence of tasks using rewards in an unknown environment while balancing exploration and exploitation.
Motivated by the embodied intelligence scenario, we suppose that tasks differ in rewards, but share the same state and action spaces and the transition dynamics~\cite{xie2021lifelong}. 
To be realistic, we make \emph{no} assumptions on how the tasks and initial states are selected.\footnote{We adopt a stricter definition of lifelong RL here to distinguish it from multi-task RL, while there are existing works on lifelong RL (e.g. \cite{brunskill2014pac,lecarpentier2021lipschitz}) assuming i.i.d. tasks.}
Generally, we allow them to be chosen from a continuous set 
potentially by an adversary based on the agent's past behaviors. 
Once a task is specified, the agent has one chance to complete the task and then the next task is revealed.

The goal of the agent is to perform near optimally for the tasks it faces, despite the online nature of the problem.
For simplicity, we assume that there is no memory constraint; this is usually the case for robotics applications where real-world interactions are the main bottleneck~\cite{xie2021lifelong}. Nonetheless, the agent should eventually learn to make decisions without frequent deliberate planning, because planning is time consuming and creates undesirable wait time for user-interactive scenarios. In other words, the agent needs to learn a multi-task policy, generalizing from not only past samples but also past computation, to solve new tasks. 

Formally, we consider an episodic setup based on the framework of contextual Markov decision process (MDP)~\cite{abbasi2014online,hallak2015contextual}. It repeats the following steps:
\begin{enumerate*}[label=\textit{\arabic*)}]
    \item At the beginning of an episode, the agent is set to an initial state and receives a context specifying the task reward, both of which can be arbitrarily chosen.
    \item When needed, the agent uses its past experiences to plan for the current task.
    \item The agent runs a policy in the environment for a fixed horizon in an attempt to solve the assigned task and gains experiences from its policy execution.
\end{enumerate*}
The agent's performance is measured as the regret with respect to the optimal policy of the corresponding task. We require that, for \emph{any} task sequence, \emph{both} the agent's  overall regret and number of planning calls to be  sublinear in the number of episodes.
While lifelong RL is not new, the need of \emph{simultaneously} achieving
\begin{enumerate*}[label=\textit{\arabic*)}]
    \item sublinear regret and 
    \item sublinear number of planning calls for
    \item a potential adversarial sequence of tasks and initial states
\end{enumerate*} 
makes the setup considered here particularly challenging. 
To our knowledge, existing works only address a strict subset of these requirements; especially, the computation aspect is often ignored.
Most provable works in lifelong RL make the assumption that the tasks are finitely many~\cite{ammar2015safe,ammar2014online,zhan2017scalable,brunskill2015online}, or are i.i.d.~\cite{brunskill2014pac,abel2018state,abel2018policy,lecarpentier2021lipschitz}, while others considering similar setups to ours do not provide regret guarantees~\cite{isele2016using,xie2021lifelong}.
On the technical side, we found that the closest works are  \cite{modi2020no,abbasi2014online,hallak2015contextual,modi2018markov,kakade2020information} for contextual MDP and \cite{wu2021accommodating,abels2019dynamic} for the dynamic setting of multi-objective RL, which study the sample complexity of learning for arbitrary task sequences; however, they either assume the problem is tabular or require a model-based planning oracle with unknown complexity.
Importantly, none of the existing works properly address the need of sublinear number of planning calls, which creates a large gap between the abstract setup and practice need.


 
In this paper, we propose the first provably efficient lifelong RL algorithm, {\bf U}CB {\bf L}ifelong {\bf V}alue {\bf D}istillation (\UCBlvd, pronounced as ``UC Boulevard''), that possesses all three desired qualities. Under the assumption of contextual MDP with linear features~\cite{yang2019sample,jin2020provably} and a new completeness-style assumption, \UCBlvd achieves sublinear regret for any online sequence of tasks while using sublinear number of planning calls. Specifically, for $K$ episodes of horizon $H$, we prove a regret bound $\Otilde(\sqrt{(d^3+\dpr d)H^4K})$ based on $\Otilde(dH\log(K))$ number of planning calls, where $d$ and $\dpr$ are the feature dimensions of the dynamics and rewards, respectively. 

From a high-level viewpoint, \UCBlvd uses the linear structure to identify what to transfer and operates by interleaving 
\begin{enumerate*}[label=\textit{\arabic*)}]
\item 
independent planning for a set of representative task contexts and 
\item distilling the planned results into a multi-task value-based policy
\end{enumerate*}. 
\UCBlvd also constantly monitors whether the new experiences it gained is sufficiently significant, based on a doubling schedule, to avoid unnecessary planning.
The design of \UCBlvd is inspired by single-task LSVI-UCB~\cite{jin2020provably} but we introduce a novel distillation step, along with a new completeness assumption, to enable computation sharing across tasks; in addition, we extend the low-switching cost technique~\cite{abbasi2011improved,gao2021provably} for single-task RL to the lifelong setup to achieve sublinear number of planning calls.



    

\section{Preliminaries}
\paragraph{Notation.}
Throughout the paper, we use lower-case letters for scalars, lower-case bold letters for vectors, and upper-case bold letters for matrices. The Euclidean-norm of $\x$ is denoted by $\norm{\x}_2$. We denote the transpose of a vector $\x$ by $\x^\top$. For any vectors $\x$ and $\y$, we use $\langle \x,\y\rangle$ to denote their inner product. We denote the Kronecker product by $\A\otimes\B$. Let $\A\in\mathbb R^{d\times d}$ be a positive definite and $\boldsymbol \nu \in\mathbb R^d$. The weighted 2-norm of $\boldsymbol \nu$ with respect to $\A$ is defined by $\norm{\boldsymbol \nub}_\A \coloneqq \sqrt{\boldsymbol \nub^\top \A \boldsymbol \nub}$. For a positive integer $n$, $[n]$ denotes the $\{1,2,\ldots,n\}$. For a real number $\alpha$, we denote $\{\alpha\}^+=\max\{\alpha,0\}$.
Finally, we use the notation $\Otilde$ for big-O notation that ignores logarithmic factors.

\subsection{Problem Formulation}\label{sec:formulation}

We formulate lifelong RL as a regret minimization problem in contextual MDP~\cite{abbasi2014online,hallak2015contextual} with adversarial context and initial state sequences. We suppose that a context determines the reward but does not affect the dynamics.
Such a context dependency is common for the lifelong learning scenario where an embodied agent consecutively solves multiple tasks.
Below we give the formal problem definition.

\paragraph{Finite-horizon Contextual MDP.} We consider a finite-horizon contextual MDP denoted by $M=(\Sc,\Ac, \Wc, H,\Pb, r)$, where $\Sc$ is the state space, $\Ac$ is the action space, $\Wc$ is the task context space, $H$ is the horizon (length of each episode), $\Pb=\{\Pb_h\}_{h=1}^H$ are the transition probabilities, and $r=\{r_h\}_{h=1}^H$ are the reward functions.
We allow $\Sc$ and $\Wc$ to be continuous or infinitely large, while we assume $\Ac$ is finite such that $\max_{a\in\Ac}$ can be performed easily.
For $h\in[H]$, 
$r_h(s,a,w)$
denotes the reward function whose range is assumed to be in $[0,1]$, and $\Pb_h(s'|s,a)$ denotes the probability of transitioning to state $s'$ upon playing action $a$ at state $s$. In short, a contextual MDP can be viewed as an MDP with state space $\Sc\times\Wc$ and action space $\Ac$ where the context part of the state remains constant in an episode.\footnote{In general, a context-dependent dynamics would take the form $\Pb_h(s'|s,a,w)$. }
To simplify the notation, for any function $f$, we write $\mathbb{P}_h[f](s,a)\coloneqq\mathbb{E}_{s^\prime\sim\mathbb{P}_h(.|s,a)}[f(s^\prime)]$.

\paragraph{Policy and Value Functions.}
In a finite-horizon contextual MDP, a policy $\pi=\{\pi_h\}_{h=1}^H$ is a sequence where $\pi_h:\Sc\times\Wc\to\Ac$ determines the agent's action at time-step $h$. Given $\pi$, we define its state value function as
$ V_h^\pi(s,w)\coloneqq\mathbb{E}[\sum_{h^\prime=h}^H r_{h^\prime}\left(s_{h^\prime},\pi_{h^\prime}(s_{h^\prime},w),w)\vert s_h=s\right]
$
and its action-value function as $Q_h^\pi(s,a,w) \coloneqq r_h(s,a,w) + \mathbb{P}_h[V^\pi_{h+1}(.,w)](s,a)$, where $Q_{H+1}^\pi=0$.
We denote the optimal policy as $\pi_h^\ast(s,w) \coloneqq \sup_{\pi} V_h^\pi(s,w)$, and let $V_h^\ast \coloneqq V_h^{\pi^\ast}$ and $Q_h^\ast \coloneqq Q_h^{\pi^\ast}$ denote the optimal value functions. Lastly, we recall the Bellman equation of the optimal policy: 
%
\begin{align}
Q_h^\ast(s,a,w)=r_h(s,a,w)+\mathbb{P}_h[V^\ast_{h+1}(.,w)](s,a),\quad
V_h^\ast(s,w)=\max_{a\in\Ac}Q_h^\ast(s,a,w)\label{eq:bellmanforoptimal},
\end{align}

\paragraph{Interaction Protocol of Lifelong RL.}
The agent interacts with a contextual MDP $M$ in episodes. For presentation simplicity, we assume that the reward functions $r$ are known, while the transition probabilities $\Pb$ are \emph{unknown} and must be learned online; we will discuss how reward learning can be naturally incorporated in Section \ref{sec:main result}.
At the beginning of episode $k$, the agent receives a task context $w^k\in\Wc$  and is set to an initial state $s_1^k$, both of which can be adversarially chosen. The agent can use past experiences to plan for the current task, if needed. Then the agent executes its policy $\pi^k$: at each time-step $h\in[H]$, it observes the state $s_h^k$, plays an action $a_h^k = \pi_h^k(s_h^k,w^k)$, observes a reward $r_h^k \coloneqq r_h(s_h^k,a_h^k,w^k)$, and goes to the next state $s_{h+1}^k$ according to $\mathbb{P}_h(.|s_h^k,a_h^k)$.
Let $K$ be the total number of episodes. 
The agent's goal is to achieve sublinear regret, 
where the regret is defined as
\begin{align}\label{eq:regret}
    \textstyle
   R_K\coloneqq\sum_{k=1}^K V_1^\ast(s_1^k,w^k)-V_1^{\pi^k}(s_1^k,w^k).
\end{align}
As the comparator policy above (namely $\pi^*$ that defines $V_1^\ast$) also knows the task context, achieving sublinear regret implies that the agent would attain near task-specific optimal performance on average.

\subsection{Assumptions}
Throughout the paper, we rely on the following assumptions. 

\begin{myassum}[Linear MDP]\label{assum:linearMDP}
$M=(\Sc,\Ac, H,\Pb, r, \Wc)$ is a linear MDP with feature maps $\phib:\Sc\times\Ac\rightarrow \mathbb{R}^d$ and $\psib:\Sc\times\Ac\times\Wc\rightarrow \mathbb{R}^{\dpr}$. That is, for any $h\in[H]$, 
there exist a vector $\etab_h$ and $d$ measures $\mub_h\coloneqq[{\mu_h}^{(1)},\ldots,{\mu_h}^{(d)}]^\top$ over $\Sc$ such that 
 $\Pb_h(.|s,a)=\left\langle \mub_h(.), \phib(s,a)\right\rangle$ and  $r_h(s,a,w)=\left\langle\etab_h,\psib(s,a,w)\right\rangle$, for all $(s,a,w)\in\Sc\times\Ac\times\Wc$.
\end{myassum}


\begin{myassum}[Boundedness]\label{assum:boundedness} Without loss of generality, $\norm{\phib(s,a)}_2\leq 1$, $\norm{\psib(s,a,w)}_2\leq 1$, $\norm{\mub_h(\Sc)}_2\leq \sqrt{d}$, and $\norm{\etab_h}_2\leq \sqrt{\dpr}$ for all $(s,a,w,h)\in \Sc\times\Ac\times\Wc\times[H]$.
\end{myassum}

\begin{example}[Weighted Rewards]\label{examp:weightedrewards}
An interesting and common special case is  $\psib(s,a,w)=\phib(s,a)\otimes\rhob(w)$, for some mapping $\rhob:\Wc\rightarrow \mathbb{R}^m$. In this case, it holds that $\dpr=md$ and $r_h(s,a,w)=\left\langle\rhob(w),\rb_h(s,a)\right\rangle$, where 
$\rb_h(s,a) = \A_h\phib(s,a)\in \mathbb{R}^m$, for some $\A_h\in\mathbb{R}^{m\times d}$, is the vector reward functions at time-step $h$. 
We can view $r_h(s,a,w)$ as a weighted reward with weights $\rhob(w)$ that depend on task $w$. This setting is closely related to Multi-Objective RL studied for tabular case in \cite{wu2021accommodating}, which studies the case where $\rhob(w)=w\in\mathbb{R}^m$ along with tabular $\Sc$ and $\Ac$.
\end{example}


\section{A Warm-up Algorithm for Lifelong RL}\label{sec:LLLSVI}

We first present 
a warm-up algorithm, termed Lifelong Least-Squares Value Iteration (\LLLSVI), in Algorithm \ref{alg:LLLSVI}.
\LLLSVI extends the single-task LSVI-UCB algorithm proposed by \cite{jin2020provably} to the lifelong learning setting. It runs LSVI-UCB as a subroutine for each task by leveraging the structure in Assumption \ref{assum:linearMDP}: as the transition dynamics is context independent, dynamics samples collected during solving other tasks can be relabeled with the current reward to plan for the current task. 
The motivation of this warm-up algorithm is to give intuitions on how the problem structure in Assumption \ref{assum:linearMDP} can be used to achieve small regret in lifelong learning.
We will show that \LLLSVI has a sublinear regret bound, which matches the minimax optimal rate in the special case studied by \cite{wu2021accommodating} in terms of number of objectives, $m$ (see Example \ref{examp:weightedrewards}).

However, we will also show that \LLLSVI is not computationally efficient, in the sense that the number of planning calls it requires grows linearly with the number of episodes. This is because the agent never learns to internalize the task solving skills but requires going though all past experiences for planning every time a new task arrives.
Moreover, we will discuss why it cannot be made computationally efficient in an easy manner.
This drawback motivates our main algorithm, \UCBlvd, in Section \ref{sec:UCBlvd}, which is provably efficient in terms of both regret and number of planning calls.

\begin{algorithm}[t]
\DontPrintSemicolon
\KwInput{$\Ac$, $\la$, $\delta$, $H$, $K$, $\beta$}
{\bf Set:} $Q_{H+1}^k(.,.,.)=0,~\forall k\in[K]$\;
 \For{{\rm episodes} $k=1,\ldots,K$}
  {
  Observe the initial state $s_1^k$ and the task context $w^k$.\;
      \For{{\rm time-steps} $h=H,\ldots,1$\label{Line:alg1firstloopbegin} }
      {
    Compute $\tilde\thetab_h^k(w^k)$ as in \eqref{eq:tildethetahk} using $Q_{h+1}^k$ defined in \eqref{eq:Q1}.\label{Line:alg1firstloopend}
  
      }
     \For{{\rm time-steps} $h=1,\ldots,H$}{
      Compute $Q_h^k(s_h^k,a,w^k)$ for all $a\in\Ac$ as in \eqref{eq:Q1}.\label{line:alg1currentQ}\;
      Play $a_h^k=\argmax_{a\in \Ac}Q_h^{k}(s_h^k,a,w^k)$ and observe $s_{h+1}^k$ and $r_h^k$.\label{line:alg1decisionrule}
      }
  }
 \caption{\LLLSVI}
  \label{alg:LLLSVI}
\end{algorithm}

\subsection{Algorithmic Notations}

To begin, we introduce the template and the notations that will be used commonly in presenting the warm-up algorithm, \LLLSVI, and our main algorithm, \UCBlvd.
For each algorithm, first we will define an algorithm-specific action-value function  $Q_h^k:\Sc\times\Ac\times\Wc\rightarrow\mathbb{R}$, which determines the agent's policy at time-step $h$ in episode $k$; then we present the full algorithm and its analyses using the quantities below, which are defined with respect to each algorithm's definition of $Q_h^k$. 

Given $\{Q_h^k\}_{h\in[H]}$,  we define state value functions and their backups as
\begin{align}
   V_h^k(s,w)&\coloneqq\min\left\{\max_{a\in\Ac}Q_h^k(s,a,w),H\right\}\label{eq:Vhk},\\
   \thetab_h^k(w)&\coloneqq\int_{\Sc}V_{h+1}^k(s^\prime,w)d\mub_h(s^\prime)\label{eq:thetahk},
\end{align}
Thanks to the linear MDP structure in Assumption \ref{assum:linearMDP}, 
it holds that 
\begin{align}\label{eq:PVh+1linearform}
    \mathbb{P}_h\left[V_{h+1}^k(.,w)\right](s,a) = \left\langle\thetab_{h}^k(w),\phib(s,a)\right\rangle.
\end{align}
Let $\la>0$ be a constant.
We define the $\la$-regularized least squares estimator of $\thetab_h^k(w)$ as
\begin{align}
\tilde\thetab_h^k(w) &\coloneqq  \left(\Lambdab_h^k\right)^{-1}\sum_{\tau=1}^{k-1}\phib_h^\tau
    V_{h+1}^k(s_{h+1}^\tau,w)
   \label{eq:tildethetahk}\\
    \Lambdab_h^k &\coloneqq \la \Iden_{d}+\sum_{\tau=1}^{k-1}\phib_h^\tau{\phib_h^\tau}^\top.\label{eq:matrixLambda}
\end{align}
where $\tilde\thetab_h^k(w)$ is the solution to
$\min_{\thetab\in\mathbb{R}^d}\sum_{\tau=1}^{k-1}(\langle\thetab,\phib(s_h^\tau,a_h^\tau)\rangle-V_{h+1}^k(s_{h+1}^\tau,w))^2+\la\norm{\thetab}_2^2$, $\phib_h^\tau\coloneqq \phib(s_h^\tau,a_h^\tau)$, 
and $\Iden_{d}\in\mathbb{R}^{d\times d}$ is the identity matrix.


\subsection{Details of Lifelong-LSVI and Theoretical Analysis}\label{sec:overview1}

We define the upper confidence bound (UCB) style action-value function of \LLLSVI as follows:
\begin{align}\label{eq:Q1}
   Q_h^k(s,a,w) \coloneqq r_h(s,a,w)+ \left\langle\tilde\thetab_h^k(w),\phib(s,a)\right\rangle+\beta\norm{\phib(s,a)}_{(\Lambdab_h^k)^{-1}},
\end{align}
where $Q_{H+1}^k(.,.,.) = 0$ and $\tilde\thetab_h^k(w)$ and $\Lambdab_h^k$ are defined in \eqref{eq:tildethetahk} and \eqref{eq:matrixLambda}, respectively. Here, $\beta$ is an exploration factor that is passed as an input of \LLLSVI and will be appropriately chosen in Theorem \ref{thm:regret}. At episode $k$, given $w^k$, \LLLSVI first performs planning backward in time based on past data to compute $\tilde\thetab_h^k(w^k)$ in \eqref{eq:tildethetahk} using $Q_{h+1}^k$ defined in \eqref{eq:Q1} (Lines \ref{Line:alg1firstloopbegin}- \ref{Line:alg1firstloopend}). Then, in execution, it uses $\tilde\thetab_h^k(w^k)$ to compute $Q_h^k(s_h^k,a,w^k)$ for the current state and all $a\in\Ac$ (Line \ref{line:alg1currentQ}) and executes the action with the highest value (Line \ref{line:alg1decisionrule}).

We show that \LLLSVI achieves sublinear regret for our lifelong RL setup. The complete proof is reported in Appendix \ref{sec:proofofLLLSVI}, which follows the ideas of LSVI-UCB~\cite{jin2020provably}.
\begin{theorem}\label{thm:regret}
Let $T=KH$. Under Assumptions \ref{assum:linearMDP} and \ref{assum:boundedness}, there exists an absolute constant $c>0$ such that for any fixed $\delta\in(0,0.5)$, if we set $\la=1$ and $\beta=cH\left(d+\sqrt{\dpr}\right)\sqrt{\log( d\dpr T/\delta)}$ in Algorithm \ref{alg:LLLSVI}, then with probability at least $1-2\delta$, it holds that
\begin{align}
    R_K\leq 2H\sqrt{T\log(d T/\delta)} +2H\beta\sqrt{2d K\log(K)}\leq \Otilde\left(\sqrt{(d^3+d\dpr) H^3T}\right).\nn
\end{align}
\end{theorem}
Before moving forward to our main algorithm in Section \ref{sec:UCBlvd}, we make a few remarks on the regret and number of planning calls of \LLLSVI. First, 
Theorem \ref{thm:regret} implies that for the special case studied by \cite{wu2021accommodating} (summarized in Example \ref{examp:weightedrewards}), the regret bound of \LLLSVI becomes $\Otilde(\sqrt{md^3H^3T})$. This rate is optimal in terms of its dependency on $m$, as shown in \cite{wu2021accommodating}, for this specific reward structure. 
Furthermore, this rate matches the regret dependencies on $d$ and $H$ of LSVI-UCB's for the single-task setting \cite{jin2020provably}.


While \LLLSVI has a decent regret guarantee, we observe that it requires computing $\tilde\thetab_h^k(w^k)$ for all $h\in[H]$, whenever a distinct new task $w^k$ arrives. Since the number of unique tasks may be as large as $K$, the total number of planning calls required in \LLLSVI is $K$ in the worst case. Unfortunately, the number of planning calls of \LLLSVI cannot be easily improved due to the nonlinear dependency of $Q_h^k(s,a,w)$ on $w$ through $\tilde\thetab_h^k(w)$  in \eqref{eq:Q1}, which could lead to a covering number no less than $K$ in general.
In particular, it is also hopeless to employ low switching cost techniques like~\cite{abbasi2011improved} to reduce the number of planning calls, because we always need to recalculate $\tilde\thetab_h^k(w)$ for every new task.

In the next section, we discuss how placing a completeness-style assumption would help circumvent the issue of non-linear dependency of the action-value functions on $w$, and consequently would enable computation sharing to decrease the number of planning calls to $\Oc({d}H\log\left(1+K/d\la\right))$.
\section{UCB Lifelong Value Distillation (\UCBlvd)}\label{sec:UCBlvd}

\begin{algorithm}[t]
\DontPrintSemicolon
\KwInput{$\Ac$, $\la$, $\delta$, $H$, $K$, $\beta$}
$\bf Set:$ $Q_{H+1}^k(.,.,.)=0,~\forall k\in[K]$, $\tilde k = 1$\;
 \For{{\rm episodes} $k=1,\ldots,K$}
  {
  Observe the initial state $s_1^k$ and the task context $w^k$.\;
\If{$\exists h\in[H]$ such that 
$\log{\det \Lambdab_h^k}- \log{\det \Lambdab_h^{\tilde k}}>1$
\label{line:criteria}}{
$\tilde k = k$\label{line:ktildeupdated}\;
\For{{\rm time-steps} $h=H,\ldots,1$}{
  Compute $\hat\xib_h^{\tilde k}$ as in \eqref{eq:hatxibandhattheta}.
      }
      }
      \For{{\rm time-steps} $h=1,\ldots,H$}{
      Compute $Q_h^{\tilde k}(s_h^k,a,w^k)$ for all $a\in\Ac$ as in \eqref{eq:Q2}.\label{line:xihat}\;
      Play $a_h^k=\argmax_{a\in \Ac}Q_h^{\tilde k}(s_h^k,a,w^k)$ and observe $s_{h+1}^k$ and $r_h^k$.\label{line:UCBlvd-decision}
      }
  }
 \caption{\UCBlvd (UCB Lifelong Value Distillation)}
  \label{alg:UCBlvd}
\end{algorithm}

In this section, we present our main algorithm, {\bf UCB} {\bf L}ifelong {\bf V}alue {\bf D}istillation (\UCBlvd), in Algorithm \ref{alg:UCBlvd}.
Under extra structural assumptions we will introduce in Section \ref{sec:extra assumptions}, 
\UCBlvd shares the same regret bound as \LLLSVI but reduces the number of planning calls to be sublinear.
In contrast to \LLLSVI which learns individual action-value function for each $w^k$ using $\phib(s,a)$, \UCBlvd learns a single action-value function for all $w\in\Wc$ based on $\psib(s,a,w)$ to enable computation sharing across tasks.
In general, directly extending \LLLSVI to use feature $\psib(s,a,w)\in\mathbb{R}^{d'}$ with $d'\geq d$  would increase the regret from that with $\phib(s,a)\in\mathbb{R}^d$, because the latter can exploit the context-independent dynamics structure.
\UCBlvd maintains the same order of regret as \LLLSVI by separating the planning into a novel two-step process: 
\begin{enumerate*}[label=\textit{\arabic*)}]
\item 
independent planning with $\phib$ for a set of representative task contexts and 
\item distilling the planned results into a multi-task value function parameterized by $\psib$
\end{enumerate*}. 
In addition, \UCBlvd runs a doubling schedule to decide whether replanning is necessary, which makes the total number of planning calls  sublinear.
Below we give the details of \UCBlvd.


\subsection{Enabling Computation Sharing} \label{sec:extra assumptions}

First, we introduce two extra assumptions needed by \UCBlvd to share computation across tasks. We will discuss how these assumptions can be relaxed in Section~\ref{sec:main result}.

The first assumption is a new completeness-style assumption. 
\begin{myassum}[Completeness]\label{assum:comp} Given feature maps $\phib:\Sc\times\Ac\rightarrow \mathbb{R}^{d}$ and $\psib:\Sc\times\Ac\times\Wc\rightarrow \mathbb{R}^{\dpr}$ in Assumption~\ref{assum:linearMDP}, consider the function class
\begin{small}
\begin{align}
    \mathcal{F} = &\left\{f: f(s,w) \hspace{-0.3mm}=\hspace{-0.3mm} \min\left\{\max_{a\in\Ac}\left\{ \langle\nub,\psib(s,a,w)\rangle+\beta\norm{\phib(s,a)}_{\Lambdab^{-1}}\right\}^+\hspace{-1mm}, H\right\},\nub\in\mathbb{R}^{\dpr},\Lambdab\in\mathbf{S}^{d}_{++},\beta\in\mathbb{R}\right\}.\nn
\end{align}
\end{small}
For any $f\in\Fc$ and $h\in[H]$, there exists a vector $\xib_h^{f}\in\mathbb{R}^{\dpr}$ with $\norm{\xib_h^f}\leq H\sqrt{\dpr}$ such that 
\begin{align}
    \mathbb{P}_h\left[f(.,w)\right](s,a)=\langle\xib_h^f,\psib(s,a,w)\rangle.\nn
\end{align}
\end{myassum}
It says the backup of functions in $\mathcal{F}$ should be captured by the feature $\psib$ with bounded parameters. The definition of $\mathcal{F}$ models closely the structure of action-value function used by \LLLSVI in \eqref{eq:Q1}, except $\langle\tilde\thetab_h^k(w),\phib(s,a)\rangle$ there is replaced by functions linear in $\psib(s,a,w)$. We will see that the action-value function used by \UCBlvd defined in the next section is contained in $\mathcal{F}$.

We introduce an extra structure on $\psib$ inspired by Example \ref{examp:weightedrewards}.
\begin{myassum}[Mappings]\label{assum:mapping}
We assume $\psib(s,a,w)=\phib(s,a)\otimes\rhob(w)$, for some mapping $\rhob:\Wc\rightarrow \mathbb{R}^m$, i.e., $\dpr=md$.
We assume that there is a known set $\{w^{(1)},w^{(2)},\ldots,w^{(n)}\}$ of $n\leq m$ task contexts such that $\rhob(w)\in{\rm{Span}}(\{\rhob(w^{(j)})\}_{j\in[n]})$ for all $w\in\Wc$. That is, for any $w\in\Wc$, there exist coefficients  $\{c_j(w)\}_{j\in[n]}$ such that
\begin{align}\label{eq:cprimecoefficient}
   \rhob(w) = \sum_{j\in[n]}c_j(w)\rhob(w^{(j)}).
\end{align}
We assume $\sum_{j\in[n]}\abs{c_j(w)}\leq L$ for all $w\in\Wc$ and some $L<\infty$\footnote{Such set $\{\rhob(w^{(j)})\}_{j\in[n]}$  always exists for finite-dimensional problems. We assume that this set is known to the algorithm.}.
\end{myassum}


\subsection{Details of \UCBlvd}\label{sec:overview2}
We define the UCB style action-value function of \UCBlvd as follows:
\begin{align}\label{eq:Q2}
   Q_h^k(s,a,w) \coloneqq \left\{r_h(s,a,w)+ \left\langle\hat\xib_h^k,\psib(s,a,w)\right\rangle+2L\beta\norm{\phib(s,a)}_{(\Lambdab_h^k)^{-1}}\right\}^+,
\end{align}
The parameter $\hat\xib_h^k$ is computed by solving the convex quadratically constrained quadratic program (QCQP) in \eqref{eq:hatxibandhattheta} below, which is defined on a set of representative task contexts $\{w^{(1)},w^{(2)},\ldots,w^{(n)}\}$ in Assumption \ref{assum:mapping} and state-action pairs $\Dc\coloneqq\left\{(s,a):\phib(s,a)~\text{are }~d~\text{linearly independent vectors}.\right\}$.
\begin{align}
\hat\xib_h^k, \{\hat\thetab_h^{k(j)}\}_{j\in[n]} = &\argmin_{\xib,\left\{\thetab^{(j)}\right\}_{j\in[n]}}\sum_{j\in[n]}\sum_{(s,a)\in\Dc}\left(\langle\thetab^{(j)},\phib(s,a)\rangle-\langle\xib,\psib(s,a,w^{(j)})\rangle\right)^2\label{eq:hatxibandhattheta}\\
&\text{s.t.}~\norm{\thetab^{(j)}-\tilde\thetab_h^k(w^{(j)})}_{\Lambdab_h^k}\leq\beta,~\forall j\in[n]\quad\text{and}\quad\norm{\xib}_2\leq H\sqrt{md}\nn,
\end{align}
where  $\tilde\thetab_h^k(w)$ and $\Lambdab_h^k$ are defined in \eqref{eq:tildethetahk} and \eqref{eq:matrixLambda}, respectively. 
We will show later in Lemma~\ref{lemm:UCB2} that the action-value function in \eqref{eq:Q2} is  an optimistic estimate of the optimal action-value function.. 

\UCBlvd also uses the linear dependency of $Q_h^k$ on $\psib$ to reduce calls of the planning step in \eqref{eq:hatxibandhattheta}. 
The agent triggers replanning only when it has gathered enough new information
compared to the last update at episode $\tilde k$.
This is measured by tracking the variations in the gram matrices $\{\Lambdab_h^k\}_{h\in[H]}$ (Line \ref{line:criteria} for Algorithm~\ref{alg:UCBlvd}).
Finally, when executing the policy at episode $k$, the agent chooses the action according to $Q_h^{\tilde k}$ in Line \ref{line:UCBlvd-decision}.



\subsection{Theoretical Analysis of \UCBlvd}\label{sec:main result}
We present the main theoretical result which shows \UCBlvd achieves sublinear regret in lifelong RL using sublinear number of planning calls, for any sequence of tasks.
\begin{theorem}\label{thm:regretcs}
Let $T=KH$. Under Assumptions \ref{assum:linearMDP}, \ref{assum:boundedness}, \ref{assum:comp}, and \ref{assum:mapping}, the number of planning calls in Algorithm \ref{alg:UCBlvd} is at most ${d}H\log(1+\frac{K}{d\la})$, and there exists an absolute constant $c>0$ such that for any fixed $\delta\in(0,0.5)$, if we set $\la=1$ and $\beta=cH(d+\sqrt{md})\sqrt{\log( mdT/\delta)}$ in Algorithm \ref{alg:UCBlvd}, then with probability at least $1-2\delta$, it holds that
\begin{align}
    R_K\leq  2H\sqrt{T\log({d}T/\delta)} +8HL\beta\sqrt{2{d}K\log(K)}\leq \Otilde\left(L\sqrt{(d^3+md^2)H^3T}\right).\nn
\end{align}
\end{theorem}
Theorem \ref{thm:regretcs} shows that \UCBlvd has the same regret bound as \LLLSVI in Theorem \ref{thm:regret}, but reduces the number of planning calls from $K$ to ${d}H\log(1+\frac{K}{d\la})$. As we discussed before, this is made possible by the unique QCQP-based distillation step of \UCBlvd in \eqref{eq:hatxibandhattheta}. If we were to simply perform least-squares regression to fit $\langle \psib(s,a,w), \hat\xib_h^k \rangle$ to $\{ \langle \phib(s,a), \tilde\thetab_h^k (w^{(j)})\}_{j\in[n]}$ for distillation, we cannot guarantee the required optimism, because $\langle \phi(s,a), \tilde\thetab_h^k (w)\rangle$ computed based on finite samples can be an irregular function that cannot be modelled by $\psib(s,a,w)$.

\begin{remark} 
We can extend our results to learn unknown rewards, i.e., $\etab_h$ in Assumption~\ref{assum:linearMDP}. This can be done by introducing a slightly different completeness assumption with an additional exploration bonus in terms of $\psib$, and then combining tools from linear bandits \cite{abbasi2011improved} and our analysis for proving Theorem \ref{thm:regretcs}.
Because reward learning affects the radius of our high probability confidence intervals for $\thetab_h^k(w)$, the number of planning calls and regret would increase by factors of $\Oc(m)$ and $\Oc(\sqrt{m})$
\footnote{While for both settings in this remark and Remark \ref{remark:bonusinpsi}, the action-value functions contain exploration bonus in terms of $\psib$, the regret here is better by a factor of $\sqrt{m}$ and this is because the multiplicative factor $\beta$ here saves a factor $\sqrt{m}$ compared to that in Remark \ref{remark:bonusinpsi}}, respectively, compared to those in Theorem \ref{thm:regretcs}. See Appendix \ref{sec:unknonwreward} for details. 
\end{remark}
\begin{remark}
It is possible to eliminate the assumption that $\psib(s,a,w)=\phib(s,a)\otimes\rhob(w)$. In this case, our analysis requires a set $\{w^{(1)},w^{(2)},\ldots,w^{(n)}\}$ of $n$ tasks such that $\psib(s,a,w)\in{\rm{Span}}(\{\psib(s,a,w^{(j)})\}_{j\in[n]})$ for all $(s,a,w)\in\Sc\times\Ac\times\Wc$. 
In Appendix \ref{sec:relaxedmapping}, we provide details of this relaxation, and show that the corresponding modified version of \UCBlvd still enjoys planning calls and regret of the same order as those of \UCBlvd. 
\end{remark}

\begin{remark}\label{remark:bonusinpsi}
We can eliminate Assumptions \ref{assum:linearMDP} and \ref{assum:mapping}
and instead design a computation-sharing version of \LLLSVI  by a sightly different completeness assumption with an exploration bonus $\beta\norm{\psib(s,a,w)}_{\tilde\Lambdab^{-1}}$. This version would use $
Q_h^k(s,a,w) \coloneqq \{r_h(s,a,w)+\langle\tilde\nub_h^k,\psib(s,a,w)\rangle+\beta\norm{\psib(s,a,w)}_{(\tilde\Lambdab_h^k)^{-1}}\}^+$, where $\tilde\nub_h^k = (\tilde\Lambdab_h^k)^{-1}\sum_{\tau=1}^{k-1}\psib_h^{\tau}.\min\{\max_{a\in\Ac}Q_{h+1}^k(s_{h+1}^\tau,a,w^\tau),H\}$, $\tilde\Lambdab_h^k = \la \Iden_{d^\prime}+\sum_{\tau=1}^{k-1}\psib_h^{\tau}{\psib_h^{\tau}}^\top$, $\psib_h^\tau=\psib(s_h^\tau,a_h^\tau,w^\tau)$, and $\beta = \Otilde(\dpr)$. In Appendix \ref{sec:standardLSVI}, we show how this change results in $\Otilde(mdH)$ number of planning calls and a regret scaling with $\Otilde(\sqrt{m^3d^3})$ for settings with $\psib(s,a,w)=\phib(s,a)\otimes\rhob(w)$. These are worse than the number of planning calls and regret in Theorem \ref{thm:regretcs} of \UCBlvd by a factor of $\Oc(m)$.
\end{remark}

\subsection{Proof Sketch of Theorem \ref{thm:regretcs}}

The complete proof of Theorem \ref{thm:regretcs} is reported in Appendix \ref{sec:proofofUCBlvd}. Here we provide a sketch. Because the proof for the bound on the number of planning calls follows standard arguments in low switching cost analysis~\cite{abbasi2011improved}, in this section, we focus on the proof sketch for the regret bound. We start by introducing the following lemma of a high probability event $\Ec_1$, which is the foundation of the analysis.
\begin{lemma}\label{lemm:bellmanupdate2} Follow the setting of Theorem \ref{thm:regretcs}.
The event
\begin{align}\label{eq:event1}
    \Ec_1(w)\coloneqq\left\{\norm{\thetab^k_h(w)-\tilde\thetab^k_h(w)}_{\Lambdab_h^k}\leq \beta,\forall(h,k)\in [H]\times[K]\right\}.
\end{align}
holds with probability at least $1-\delta$ for a fixed $w$.
\end{lemma}
The following lemma highlights the importance of the carefully designed planning step in \eqref{eq:hatxibandhattheta}. In particular, it emphasizes how this step paired with the choice of set $\Dc$, Assumptions \ref{assum:comp} and \ref{assum:mapping} leads to good estimators for $\xib_h^{V_{h+1}^\ast}$, 
without the need of the bonus term $\norm{\psib(s,a,w)}_{\left(\tilde\Lambdab_h^k\right)^{-1}}$ that the alternate extension of \LLLSVI in Remark \ref{remark:bonusinpsi} has. This step saves a factor of $\Oc(m)$ in the number of planning calls and regret.

\begin{lemma}\label{lemm:mainkeylemma}
Let $\widetilde\Wc=\{w^\tau:\tau\in[K]\}\cup\{w^{(j)}:j\in[n]\}$. Conditioned on events $\{\Ec_1(w)\}_{w\in\widetilde\Wc}$ defined in \eqref{eq:event1}, for all $(s,a,w,h,k)\in\Sc\times\Ac\times\widetilde\Wc\times[H]\times[K]$, it holds that
\begin{align}
    \abs{\langle\hat\xib_h^k,\psib(s,a,w)\rangle-\mathbb{P}_h[V_{h+1}^k(.,w)](s,a)}\leq 2L\beta\norm{\phib(s,a)}_{\left(\Lambdab_h^k\right)^{-1}}.\nn
\end{align}
\end{lemma}

As the final step in the regret analysis, we state the following lemma which uses Lemma \ref{lemm:mainkeylemma} to prove the optimistic nature of \UCBlvd.  Then following the standard analysis of single-task LSVI-UCB we derive the regret bound in Theorem \ref{thm:regretcs}.
\begin{lemma}\label{lemm:UCB2}
    Let $\widetilde\Wc=\{w^\tau:\tau\in[K]\}\cup\{w^{(j)}:j\in[n]\}$. Conditioned on events $\{\Ec_1(w)\}_{w\in\widetilde\Wc}$ defined in \eqref{eq:event1}, and with $Q_h^k$ computed as in \eqref{eq:Q2}, it holds that $Q_h^k(s,a,w)\geq Q_h^\ast(s,a,w)$ for all $(s,a,w,h,k)\in\Sc\times\Ac\times\widetilde\Wc\times[H]\times[K]$.
   \end{lemma}



\section{Related Work}


We consider the regret minimization setup of lifelong RL under the contextual MDP framework, where the agent receives tasks specified by contexts in sequence and needs to achieve a sublinear regret for any task sequence. 
Below, we contrast our work with related work in the literature.

\paragraph{Lifelong RL} Generally lifelong RL studies how to learn to solve a streaming sequence of tasks using rewards.
While it was originally motivated by the need of endless learning of robots~\cite{thrun1995lifelong}, historically many works on lifelong RL~\cite{brunskill2014pac,abel2018state,abel2018policy,lecarpentier2021lipschitz} assume that the tasks are i.i.d. (similar to multi-task RL; see below). 
There are works for adversarial sequences, but most of them assume finite number of tasks \cite{ammar2014online,brunskill2015online,ammar2015safe,zhan2017scalable} or are purely empirical~\cite{xie2021lifelong}. The work by \cite{isele2016using} uses contexts to enable zero-shot learning like here, but it (as well as most works above) do not provide formal regret guarantees.\footnote{\cite{ammar2015safe} give regret bounds but only for linearized value difference; \cite{brunskill2015online} show regret bounds only for finite number of tasks.} \cite{brunskill2015online,xie2021lifelong} assume the task identity is latent, which requires additional exploration; in this sense, their problem is harder than the setup here where the task context is revealed. Extending the setup here to consider latent context is an important future research direction.

\paragraph{Contextual MDP and Multi-objective RL}
Our setup is closely related to the exploration problem studied in the contextual MDP literature, though contextual MDP is originally not motivated from the lifelong learning perspective. A similar mathematical problem appears in the dynamic setup of multi-objective RL~\cite{wu2021accommodating,abels2019dynamic}, which can be viewed as a special case of contextual MDP where the context linearly determines the reward function but not the dynamics.
Most contextual MDP works allow adversarial contexts and initial states, but a majority of them focuses on the tabular setup~\cite{abbasi2014online,hallak2015contextual,modi2018markov,modi2020no,levy2022learning,wu2021accommodating}, whereas our setup allows continuous states. \cite{kakade2020information,du2019continuous} allow continuous state and actions, but the former assumes a planning oracle with unclear computational complexity and the latter focuses on only LQG problems.
While generally contextual MDP allows both the reward and the dynamics to vary with contexts, we focus on the effects of context-independent dynamics similar to~\cite{kakade2020information,wu2021accommodating}. In particular, the recent work of \cite{wu2021accommodating} is the closest to ours, but they study the sample complexity in the tabular setup with linearly parameterized rewards. In view of Example \ref{examp:weightedrewards}, their proposed algorithm has a regret bound  $\Otilde(\sqrt{\min\{m,\abs{\Sc}\}H\abs{\Sc}\abs{\Ac}K})$. 
However, they need linear number of planning calls. 
On the contrary, our algorithm, \UCBlvd, allows continuous states, nonlinear context dependency, and has both sublinear regret and number of planning calls.


\paragraph{Multi-Task RL} 
Another closely related line of work is multi-task RL. 
Compared to our setting, multi-task RL assumes that there are beforehand known finite tasks and/or they are i.i.d .samples from a fixed distribution. For example, in \cite{yang2020multi,hessel2019multi,brunskill2013sample,fifty2021efficiently,zhang2021provably,sodhani2021multi}, tasks are assumed to be chosen from a known finite set, and in \cite{yang2020multi,wilson2007multi,brunskill2013sample,sun2021temple}, tasks are sampled from a fixed distribution. By contrast, our setting provides guarantees on regret and number of planning calls for adversarial task sequences.



\section{Discussion}\label{sec:discuss}
In this paper, we make a link between lifelong RL and contextual MDPs. We propose \UCBlvd, an algorithm that \emph{simultaneously} satisfies the need of achieving
\begin{enumerate*}[label=\textit{\arabic*)}]
    \item sublinear regret and 
    \item sublinear number of planning calls for
    \item a potential adversarial sequence of tasks and initial states
\end{enumerate*}. Specifically, for $K$ task episodes of horizon $H$, we proved that \UCBlvd has a regret bound $\Otilde(\sqrt{(d^3+\dpr d)H^4K})$  based on $\Otilde(dH\log(K))$ number of planning calls, where $d$ and $\dpr$ are the feature dimensions of the dynamics and rewards, respectively. We believe that our results would inspire several research directions in the literature of CMDP and multi-objective RL, as existing work to our knowledge does not cover the computation complexity sharing aspect.
That said, our work's limitations motivate further investigations in the following directions: 
\begin{enumerate*}[label=\textit{\arabic*)}]
    \item extension to more general class of MDPs, potentially using general function approximation tools,
    \item establishing an information-theoretic lower bound on the number of planning calls/computation complexity.
\end{enumerate*}

\section*{Acknowledgement}
This work was supported in part by DARPA grant HR00112190130.



\newpage
\bibliographystyle{apalike}
\bibliography{main}



\newpage
\appendix

\section{Proofs of Section \ref{sec:LLLSVI}}\label{sec:proofofLLLSVI}

To prove Theorem \ref{thm:regret}, we will use the high probability event $\Ec_2$ defined in Lemma \ref{lemm:bellmanupdate1} to prove the UCB nature of \LLLSVI in Lemma \ref{lemm:UCB1}, which is the key to controlling the regret. We first state the following lemma that will be used in the proof of Lemma \ref{lemm:bellmanupdate1}.
\begin{lemma}\label{lemm:event2lemma}
Under the setting of Theorem \ref{thm:regret}, let $c_\beta$ be the constant in the definition of $\beta$. Then, for a fixed $w$, there is an absolute constant $c_0$ independent of $c_\beta$, such that for all $(h,k)\in[H]\times[K]$, with probability at least $1-\delta$ it holds that
\begin{align}
     \norm{\sum_{\tau=1}^{k-1}\phib_h^\tau.\left( V_{h+1}^k(s_{h+1}^\tau,w)-\mathbb{P}_h[V^k_{h+1}(.,w)](s_h^\tau,a_h^\tau)\right)}_{\left(\Lambdab_h^k\right)^{-1}}\leq c_0H\left(d+\sqrt{\dpr}\right)\sqrt{\log((c_\beta+1) d\dpr T/\delta)}\nn,
\end{align}
where $c_0$ and $c_\beta$ are two independent absolute constants.
\end{lemma}

\begin{proof}

We note that $\norm{\etab_h}_2\leq\sqrt{\dpr}$ (Assumption \ref{assum:boundedness}),  $\norm{\theta_h^k(w)}_2\leq H\sqrt{d}$ (Lemma \ref{lemm:boundonweight}), and $\norm{\left(\Lambdab_{h}^k\right)^{-1}}\leq\frac{1}{\la}$. Thus, Lemmas \ref{lemm:lemmaD.4inJinetal} and \ref{lemm:coveringnumberQ1} together imply that for all $(h,k)\in[H]\times[K]$, with probability at least $1-\delta$ it holds that
\begin{align}
     &\norm{\sum_{\tau=1}^{k-1}\phib_h^\tau\left( V_{h+1}^k(s_{h+1}^\tau,w)-\mathbb{P}_h[V^k_{h+1}(.,w)](s_h^\tau,a_h^\tau)\right)}^2_{\left(\Lambdab_h^k\right)^{-1}}\nn\\
     &\leq4H^2\left(\frac{{d}}{2}\log\left(\frac{k+\la}{\la}\right)+\dpr\log(1+4\dpr/\epsilon)+d\log(1+4Hd/\epsilon)+d^2\log\left(\frac{1+8B^2\sqrt{d}}{\la\epsilon^2}\right)+\log\left(\frac{1}{\delta}\right)\right)+\frac{8k^2\epsilon^2}{\la}.\nn
\end{align}
If we let $\epsilon = \frac{dH}{k}$ and
$\beta = c_\beta(d+\sqrt{\dpr}) H\sqrt{\log(dT/\delta)}$, then, there exists an absolute constant $C>0$ that is independent of $c_\beta$ such that
\begin{align}
   \norm{\sum_{\tau=1}^{k-1}\phib_h^\tau\left( V_{h+1}^k(s_{h+1}^\tau,w)-\mathbb{P}_h[V^k_{h+1}(.,w)](s_h^\tau,a_h^\tau)\right)}^2_{\left(\Lambdab_h^k\right)^{-1}}\leq  C(\dpr+d^2)H^2\log\left((c_\beta+1)d\dpr T/\delta\right).\nn
\end{align}
\end{proof}

\begin{lemma}\label{lemm:bellmanupdate1}Let the setting of Theorem \ref{thm:regret} holds. 
The event
\begin{align}\label{eq:event2}
    \Ec_2(w)\coloneqq\left\{\norm{\thetab^k_h(w)-\tilde\thetab^k_h(w)}_{\Lambdab_h^k}\leq \beta,\forall(h,k)\in [H]\times[K]\right\}.
\end{align}
holds with probability at least $1-\delta$ for a fixed $w$.
\end{lemma}
\begin{proof}

\begin{align}
    \thetab_h^k(w) - \tilde\thetab_h^k(w)&= \thetab_h^k(w)-\left(\Lambdab_h^k\right)^{-1}\sum_{\tau=1}^{k-1}\phib_h^\tau V_{h+1}^k(s_{h+1}^\tau,w)\nn\\
    &= \left(\Lambdab_h^k\right)^{-1}\left(\Lambdab_h^k\thetab_h^k(w)-\sum_{\tau=1}^{k-1}\phib_h^\tau V_{h+1}^k(s_{h+1}^\tau,w)\right)\nn\\
    &=\underbrace{\la\left(\Lambdab_h^k\right)^{-1}\thetab_h^k(w)}_{\qb_1}\underbrace{-\left(\Lambdab_h^k\right)^{-1}\left(\sum_{\tau=1}^{k-1}\phib_h^\tau\left( V_{h+1}^k(s_{h+1}^\tau,w)-\mathbb{P}_h[V^k_{h+1}(.,w)](s_h^\tau,a_h^\tau)\right)\right)}_{\qb_2}\nn.
\end{align}

Thus, in order to upper bound $\norm{\thetab_h^k(w)-\tilde\thetab_h^k(w)}_{\Lambdab_h^k}$, we bound $\norm{\qb_1}_{\Lambdab_h^k}$ and $\norm{\qb_2}_{\Lambdab_h^k}$ separately.

From Lemma \ref{lemm:boundonweight}, we have
\begin{align}\label{eq:q1phi1}
    \norm{\qb_1}_{\Lambdab_h^k}=\la\norm{\thetab_h^k(w)}_{\left(\Lambdab_h^k\right)^{-1}}\leq \sqrt{\la}\norm{\thetab_h^k(w)}_2\leq H\sqrt{\la d}.
\end{align}

Thanks to Lemma \ref{lemm:event2lemma}, for all $(w,h,k)$, with probability at least $1-\delta$, it holds that
\begin{align} \label{eq:q2phi1}
    \norm{\qb_2}_{\Lambdab_h^k}\leq \norm{\sum_{\tau=1}^{k-1}\phib_h^\tau\left( V_{h+1}^k(s_{h+1}^\tau,w)-\mathbb{P}_h[V^k_{h+1}(.,w)](s_h^\tau,a_h^\tau)\right)}_{\left(\Lambdab_h^k\right)^{-1}}\leq c_0H\left(d+\sqrt{\dpr}\right)\sqrt{\log((c_\beta+1) d\dpr T/\delta)},
\end{align}
where $c_0$ and $c_\beta$ are two independent absolute constants.

Combining \eqref{eq:q1phi1} and \eqref{eq:q2phi1}, for all $(w,h,k)$, with probability at least $1-\delta$, it holds that

\begin{align}
    \norm{\thetab_h^k(w)-\tilde\thetab_h^k(w)}_{\Lambdab_h^k}\leq cH\left(d+\sqrt{\dpr}\right)\sqrt{\la\log( d\dpr T/\delta)}\nn
\end{align}
for some absolute constant $c>0$.

\end{proof}

\begin{lemma}\label{lemm:UCB1}
    Let $\widetilde\Wc=\{w^1,w^2,\ldots,w^K\}$. Conditioned on events $\{\Ec_2(w)\}_{w\in\widetilde\Wc}$ defined in \eqref{eq:event2}, and with $Q_h^k$ computed as in \eqref{eq:Q1}, it holds that $Q_h^k(s,a,w)\geq Q_h^\ast(s,a,w)$ for all $(s,a,w,h,k)\in\Sc\times\Ac\times\widetilde\Wc\times[H]\times[K]$.
  \end{lemma}
  
  \begin{proof}

We first note that conditioned on events $\{\Ec_2(w)\}_{w\in\widetilde\Wc}$ , for all $(s,a,w,h,k)\in\Sc\times\Ac\times\widetilde\Wc\times[H]\times[K]$, it holds that
\begin{align}
&\abs{r_h(s,a,w)+\left\langle\tilde\thetab_h^k(w),\phib(s,a)\right\rangle-Q_h^\pi(s,a,w)-\mathbb{P}_h\left[V_{h+1}^k(.,w)-V_{h+1}^\pi(.,w)\right](s,a)}\nn\\
&=\abs{r_h(s,a,w)+\left\langle\tilde\thetab_h^k(w),\phib(s,a)\right\rangle-r_h(s,a,w)-\mathbb{P}_h\left[V_{h+1}^k(.,w)\right](s,a)}\nn\\
&=\abs{\left\langle\tilde\thetab_h^k(w),\phib(s,a)\right\rangle-\mathbb{P}_h\left[V_{h+1}^k(.,w)\right](s,a)}\nn\\
&=\abs{\left\langle\tilde\thetab_h^k(w)-\thetab_h^k(w),\phib(s,a)\right\rangle}\nn\\
&\leq \norm{\tilde\thetab_h^k(w)-\thetab_h^k(w)}_{\Lambdab_h^k}\norm{\phib(s,a)}_{\left(\Lambdab_h^k\right)^{-1}}\nn\\
&\leq\beta\norm{\phib(s,a)}_{\left(\Lambdab_h^k\right)^{-1}},\tag{Lemma \ref{lemm:bellmanupdate1}}
\end{align}
for any policy $\pi$.

Now, we prove the lemma by induction. The statement holds for $H$ because $Q_{H+1}^k(.,.,.)=Q_{H+1}^\ast(.,.,.)=0$ and thus conditioned on events $\{\Ec_2(w)\}_{w\in\widetilde\Wc}$, defined in \eqref{eq:event2}, for all $(s,a,w,k)\in\Sc\times\Ac\times\widetilde\Wc\times[K]$, we have

\begin{align}
   \abs{r_H(s,a,w)+\left\langle\thetab_H^k(w),\psib(s,a)\right\rangle-Q_H^{\ast}(s,a,w)}\leq \beta\norm{\phib(s,a)}_{\left(\Lambdab_H^k\right)^{-1}}.\nn
\end{align}
Therefore, conditioned on events $\{\Ec_2(w)\}_{w\in\widetilde\Wc}$, for all $(s,a,w,k)\in\Sc\times\Ac\times\widetilde\Wc\times[K]$, we have
\begin{align}
    Q^\ast_H(s,a,w)\leq r_H(s,a,w)+\left\langle\thetab_H^k(w),\phib(s,a)\right\rangle+\beta\norm{\phib(s,a)}_{(\Lambdab_H^k)^{-1}} = Q_H^k(s,a,w)\nn.
\end{align}

Now, suppose the statement holds at time-step $h+1$ and consider time-step $h$. Conditioned on events $\{\Ec_2(w)\}_{w\in\widetilde\Wc}$, for all $(s,a,w,h,k)\in\Sc\times\Ac\times\widetilde\Wc\times[H]\times[K]$, we have 
\begin{align}
    0&\leq r_h(s,a,w)+\left\langle\thetab_h^k(w),\phib(s,a)\right\rangle-Q_h^{\ast}(s,a,w)-\mathbb{P}_h\left[V_{h+1}^k(.,w)-V_{h+1}^{\ast}(.,w)\right](s,a)+\beta\norm{\phib(s,a)}_{\left(\Lambdab_h^k\right)^{-1}}\nn\\
    &\leq r_h(s,a,w)+\left\langle\thetab_h^k(w),\phib(s,a)\right\rangle-Q_h^{\ast}(s,a,w)+\beta\norm{\phib(s,a)}_{\left(\Lambdab_h^k\right)^{-1}}.\tag{Induction assumption}
\end{align}
Therefore, conditioned on events $\{\Ec_2(w)\}_{w\in\widetilde\Wc}$, for all $(s,a,w,h,k)\in\Sc\times\Ac\times\widetilde\Wc\times[H]\times[K]$, we have
\begin{align}
    Q^\ast_h(s,a,w)\leq r_h(s,a,w)+\left\langle\thetab_h^k(w),\phib(s,a)\right\rangle+\beta\norm{\phib(s,a)}_{\left(\Lambdab_h^k\right)^{-1}} = Q_h^k(s,a,w)\nn.
\end{align}
This completes the proof.
  
  \end{proof}


\subsection{Proof of Theorem \ref{thm:regret}}

Let $\delta_h^k=V_h^k(s_h^k,w^k) - V_h^{\pi^k}(s_h^k,w^k)$ and $\xi_{h+1}^k = \mathbb{E}\left[\delta_{h+1}^k\vert s_h^k,a_h^k\right]-\delta_{h+1}^k$. Conditioned on events $\{\Ec_2(w)\}_{w\in\widetilde\Wc}$, for all $(s,a,w,h,k)\in\Sc\times\Ac\times\widetilde\Wc\times[H]\times[K]$, we have
\begin{align}
    Q_h^k(s,a,w) - Q_h^{\pi^k}(s,a,w)&=r_h(s,a,w)+ \left\langle\thetab_h^k(w),\phib(s,a)\right\rangle- Q_h^{\pi^k}(s,a,w)+\beta\norm{\phib(s,a)}_{(\Lambdab_h^k)^{-1}}\nn\\
    &\leq \mathbb{P}_h\left[V_{h+1}^k(.,w)-V_{h+1}^{\pi^k}(.,w)\right](s,a)+2\beta\norm{\phib(s,a)}_{(\Lambdab_h^k)^{-1}}\label{eq:somethinginthemiddleQ1}.
\end{align}

Note that $\delta_h^k 
\leq Q_h^k(s_h^k,a_h^k,w^k) - Q_h^{\pi^k}(s_h^k,a_h^k,w^k)$. Thus, combining \eqref{eq:somethinginthemiddleQ1}, Lemma \ref{lemm:bellmanupdate1}, and a union bound over $\widetilde\Wc$, we conclude that for all $(h,k)\in[H]\times[K]$, with probability at least $1-\delta$, it holds that
\begin{align}
   \delta_h^k \leq  \xi_{h+1}^k+\delta_{h+1}^k+2\beta\norm{\phib(s_h^k,a_h^k)}_{(\Lambdab_h^k)^{-1}}\nn.
\end{align}
Now, we complete the regret analysis 
\begin{align}
    R_K &= \sum_{k=1}^K V_1^{\ast}(s_1^k,w^k)-V_1^{\pi^k}(s_1^k,w^k)\nn\\
    &\leq \sum_{k=1}^K V_1^{k}(s_1^k,w^k)-V_1^{\pi^k}(s_1^k,w^k) \tag{Lemma \ref{lemm:UCB1}}\\
    &=\sum_{k=1}^K \delta_1^k\nn\\
    &\leq \sum_{k=1}^K\sum_{h=1}^H\xi_{h}^k+2\beta\sum_{k=1}^K\sum_{h=1}^H\norm{\phib(s_h^k,a_h^k)}_{\left(\Lambdab_h^k\right)^{-1}}\nn\\
    &\leq  2H\sqrt{T\log(d T/\delta)} +2H\beta\sqrt{2d K\log(1+K/\la)}\nn\\
    &\leq \Otilde\left(\sqrt{\la(d^3+d\dpr)H^3T}\right)\nn.
\end{align}
The third inequality is true because of the following: we observe that $\{\xi_h^k\}$ is a martingale
difference sequence satisfying $|\xi_h^k|\leq 2H$. Thus, thanks to Azuma-Hoeffding inequality, we have
\begin{align}
    \mathbb{P}\left(\sum_{k=1}^K\sum_{h=1}^H\xi_{h}^k\leq 2H\sqrt{T\log(dT/\delta)}\right)\geq 1-\delta.\label{eq:azoma}
\end{align}
In order to bound $\sum_{k=1}^K\sum_{h=1}^H\norm{\phib_h^k}_{\left(\Lambdab_h^{k}\right)^{-1}}$, note that for any $h\in[H]$, we have
\begin{align}
    \sum_{k=1}^K\norm{\phib_h^k}_{\left(\Lambdab_h^{k}\right)^{-1}}&\leq \sqrt{K\sum_{k=1}^K\norm{\phib_h^k}^2_{\left(\Lambdab_h^{k}\right)^{-1}}}\tag{Cauchy-Schwartz inequality}\\
    &\leq\sqrt{2K\log\left(\frac{\det\left(\Lambdab_h^K\right)}{\det\left(\Lambdab_h^1\right)}\right)}\label{eq:firstine}\\
    &\leq\sqrt{2dK\log\left(1+\frac{K}{d\la}\right)}.\label{eq:secondine}
\end{align}
In inequality \eqref{eq:firstine}, we used the standard argument in regret analysis of linear bandits \cite{abbasi2011improved} (Lemma~11) as follows: 
\begin{align}\label{eq:standardarg}
    \sum_{t=1}^n \min\left(\norm{\y_t}^2_{\Vb_{t}^{-1}},1\right)\leq 2\log\frac{\det \Vb_{n+1}}{\det\Vb_{1}}\quad \text{where}\quad \Vb_n = \Vb_{1}+\sum_{t=1}^{n-1}\y_t\y^\top_t.
\end{align}
In inequality \eqref{eq:secondine}, we used Assumption \ref{assum:boundedness} and the fact that $\det(\A)=\prod_{i=1}^d \lambda_i(\A) \leq ({\rm trace}(\A)/d)^d$.


\section{Proofs of Section \ref{sec:UCBlvd}}\label{sec:proofofUCBlvd}

\subsection{Proof of Lemma \ref{lemm:bellmanupdate2}}\label{sec:proofofbellmanupdate2}
First, we state the following lemma that will be used in the proof of Lemma \ref{lemm:bellmanupdate2}.

\begin{lemma}\label{lemm:event1lemma}
Under the setting of Lemma \ref{lemm:bellmanupdate2}, let $c_\beta$ be a constant in the definition of $\beta$. Then, for a fixed $w$, there is an absolute constant $c_0$ independent of $c_\beta$, such that for all $(h,k)\in[H]\times[K]$, with probability at least $1-\delta$ it holds that
\begin{align}
     \norm{\sum_{\tau=1}^{k-1}\phib_h^\tau.\left( V_{h+1}^k(s_{h+1}^\tau,w)-\mathbb{P}_h[V^k_{h+1}(.,w)](s_h^\tau,a_h^\tau)\right)}_{\left(\Lambdab_h^k\right)^{-1}}\leq c_0H\left(d+\sqrt{md}\right)\sqrt{\log((c_\beta+1) mdT/\delta)}\nn,
\end{align}
where $c_0$ and $c_\beta$ are two independent absolute constants.
\end{lemma}
\begin{proof}


We note that $\norm{\etab_h+\hat\xib_h^k}_2\leq (1+H)\sqrt{md}$ and $\norm{\left(\Lambdab_{h}^k\right)^{-1}}\leq\frac{1}{\la}$. Thus, Lemmas \ref{lemm:lemmaD.4inJinetal} and \ref{lemm:coveringnumberQ2} together imply that for all $(h,k)\in[H]\times[K]$, with probability at least $1-\delta$ it holds that
\begin{align}
     &\norm{\sum_{\tau=1}^{k-1}\phib_h^\tau\left( V_{h+1}^k(s_{h+1}^\tau,w)-\mathbb{P}_h[V^k_{h+1}(.,w)](s_h^\tau,a_h^\tau)\right)}^2_{\left(\Lambdab_h^k\right)^{-1}}\nn\\
     &\leq4H^2\left(\frac{{d}}{2}\log\left(\frac{k+\la}{\la}\right)+md\log(1+8H\sqrt{md}/\epsilon)+d^2\log\left(\frac{1+32L^2\beta^2\sqrt{d}}{\la\epsilon^2}\right)+\log\left(\frac{1}{\delta}\right)\right)+\frac{8k^2\epsilon^2}{\la}\nn.
\end{align}
If we let $\epsilon = \frac{dH}{k}$ and
$\beta = c_\beta(d+\sqrt{md}) H\sqrt{\log(dT/\delta)}$, then, there exists an absolute constant $C>0$ that is independent of $c_\beta$ such that
\begin{align}
   \norm{\sum_{\tau=1}^{k-1}\phib_h^\tau\left( V_{h+1}^k(s_{h+1}^\tau,w)-\mathbb{P}_h[V^k_{h+1}(.,w)](s_h^\tau,a_h^\tau)\right)}^2_{\left(\Lambdab_h^k\right)^{-1}}\leq  C(md+d^2)H^2\log\left((c_\beta+1)mdT/\delta\right).\nn
\end{align}
\end{proof}

Now, we begin the formal proof of Lemma \ref{lemm:bellmanupdate2}:
\begin{align}
    \thetab_h^k(w) - \tilde\thetab_h^k(w)&= \thetab_h^k(w)-\left(\Lambdab_h^k\right)^{-1}\sum_{\tau=1}^{k-1}\phib_h^\tau V_{h+1}^k(s_{h+1}^\tau,w)\nn\\
    &= \left(\Lambdab_h^k\right)^{-1}\left(\Lambdab_h^k\thetab_h^k(w)-\sum_{\tau=1}^{k-1}\phib_h^\tau V_{h+1}^k(s_{h+1}^\tau,w)\right)\nn\\
    &=\underbrace{\la\left(\Lambdab_h^k\right)^{-1}\thetab_h^k(w)}_{\qb_1}\underbrace{-\left(\Lambdab_h^k\right)^{-1}\left(\sum_{\tau=1}^{k-1}\phib_h^\tau\left( V_{h+1}^k(s_{h+1}^\tau,w)-\mathbb{P}_h[V^k_{h+1}(.,w)](s_h^\tau,a_h^\tau)\right)\right)}_{\qb_2}\nn.
\end{align}

Thus, in order to upper bound $\norm{\thetab_h^k(w)-\tilde\thetab_h^k(w)}_{\Lambdab_h^k}$, we bound $\norm{\qb_1}_{\Lambdab_h^k}$ and $\norm{\qb_2}_{\Lambdab_h^k}$ separately.

From Lemma \ref{lemm:boundonweight}, we have
\begin{align}\label{eq:q1phi2}
    \norm{\qb_1}_{\Lambdab_h^k}=\la\norm{\thetab_h^k(w)}_{\left(\Lambdab_h^k\right)^{-1}}\leq \sqrt{\la}\norm{\thetab_h^k(w)}_2\leq H\sqrt{\la d}.
\end{align}

Thanks to Lemma \ref{lemm:event1lemma}, for all $(w,h,k)$, with probability at least $1-\delta$, it holds that
\begin{align} \label{eq:q2phi2}
    \norm{\qb_2}_{\Lambdab_h^k}&\leq \norm{\sum_{\tau=1}^{k-1}\phib_h^\tau\left( V_{h+1}^k(s_{h+1}^\tau,w)-\mathbb{P}_h[V^k_{h+1}(.,w)](s_h^\tau,a_h^\tau)\right)}_{\left(\Lambdab_h^k\right)^{-1}}\nn\\
    &\leq c_0H\left(d+\sqrt{md}\right)\sqrt{\log((c_\beta+1) mdT/\delta)},
\end{align}
where $c_0$ and $c_\beta$ are two independent absolute constants.

Combining \eqref{eq:q1phi2} and \eqref{eq:q2phi2}, for all $(h,k)\in[H]\times[K]$, with probability at least $1-\delta$, it holds that

\begin{align}
    \norm{\thetab_h^k(w)-\tilde\thetab_h^k(w)}_{\Lambdab_h^k}\leq cH\left(d+\sqrt{md}\right)\sqrt{\la\log( mdT/\delta)}\nn
\end{align}
for some absolute constant $c>0$.


\subsection{Proof of Lemma \ref{lemm:mainkeylemma}}\label{sec:proofofkeylemma}

Thanks to Assumption \ref{assum:comp} and conditioned on events $\{\Ec_1(w)\}_{w\in\widetilde\Wc}$, one set of solution for \eqref{eq:hatxibandhattheta} is $\left\{\thetab_h^k\left(w^{(j)}\right)\right\}_{j\in[n]}$ and $\xib_h^{V_{h+1}^k}$ with corresponding zero optimal objective value. Therefore, it holds that
\begin{align}
    \left\langle\hat\thetab_h^{k(j)},\phib(s,a)\right\rangle=\left\langle\hat\xib_h^k,\psib\left(s,a,w^{(j)}\right)\right\rangle,\quad\forall (j,(s,a))\in[n]\times\Dc.\label{eq:implicationofoptimizationproblem}
\end{align}

Let $\left(s^{(i)},a^{(i)}\right)$ be the $i$-th element of $\Dc$ and $\{c^\prime_i(s,a)\}_{i\in[d]}$ be the coefficients such that
\begin{align}
  \phib(s,a) = \sum_{i\in[d]}c^\prime_i(s,a)\phib\left(s^{(i)},a^{(i)}\right).\nn
\end{align}
For any triple $(s,a,j)\in\Sc\times\Ac\times[n]$, we have
\begin{align}
    \left\langle\hat\xib_h^k,\psib\left(s,a,w^{(j)}\right)\right\rangle &=\left\langle\hat\xib_h^k,\phib(s,a)\otimes\rhob\left(w^{(j)}\right)\right\rangle\nn\\
    &=\left\langle\hat\xib_h^k,\sum_{i\in[d]}c^\prime_i(s,a)\phib\left(s^{(i)},a^{(i)}\right)\otimes\rhob\left(w^{(j)}\right)\right\rangle\nn\\
    &=\sum_{i\in[d]}c^\prime_i(s,a)\left\langle\hat\xib_h^k,\psib\left(s^{(i)},a^{(i)},w^{(j)}\right)\right\rangle\tag{Assumption \ref{assum:mapping}}\\
    &= \sum_{i\in[d]}c^\prime_i(s,a)\left\langle\hat\thetab_h^{k(j)},\phib\left(s^{(i)},a^{(i)}\right)\right\rangle\tag{Eqn. \eqref{eq:implicationofoptimizationproblem}}\\
    &=\left\langle\hat\thetab_h^{k(j)},\phib(s,a)\right\rangle.\label{eq:zeroforallsanda}
\end{align}

For any $(s,a,w)\in\Sc\times\Ac\times\Wc$, it holds that
\begin{align}
     \mathbb{P}_h\left[V_{h+1}^k(.,w)\right](s,a) &= \left\langle\thetab_{h}^k(w),\phib(s,a)\right\rangle \tag{Eqn. \eqref{eq:PVh+1linearform}}\\
     &= \left\langle\xib_h^{V_{h+1}^k},\psib(s,a,w)\right\rangle\tag{Assumption \ref{assum:comp}}\\
     &= \sum_{j\in[n]}c_j(w)\left\langle\xib_h^{V_{h+1}^k},\psib\left(s,a,w^{(j)}\right)\right\rangle\tag{Eqn. \eqref{eq:cprimecoefficient}}\\
     &= \sum_{j\in[n]}c_j(w) \mathbb{P}_h\left[V_{h+1}^k\left(.,w^{(j)}\right)\right](s,a)\tag{Assumption \ref{assum:comp}}\\
     &= \sum_{j\in[n]}c_j(w) \left\langle\thetab_{h}^k\left(w^{(j)}\right),\phib(s,a)\right\rangle.\label{eq:Phlinearcombination}
\end{align}

Finally, conditioned on events $\{\Ec_1(w)\}_{w\in\widetilde\Wc}$, for all $(s,a,w,h,k)\in\Sc\times\Ac\times\widetilde\Wc\times[H]\times[K]$, it holds that
\begin{align}
    &\abs{\left\langle\hat\xib_h^k,\psib(s,a,w)\right\rangle-\mathbb{P}_h\left[V_{h+1}^k(.,w)\right](s,a)}\\
    &=\abs{\left\langle\hat\xib_h^k,\psib(s,a,w)\right\rangle- \left\langle\thetab_{h}^k(w),\phib(s,a)\right\rangle}\nn\\
    &=\abs{\sum_{j\in[n]}c_j(w)\left(\left\langle\hat\xib_h^k,\psib\left(s,a,w^{(j)}\right)\right\rangle- \left\langle\thetab_{h}^k\left(w^{(j)}\right),\phib(s,a)\right\rangle\right)}\tag{Eqns. \eqref{eq:cprimecoefficient} and \eqref{eq:Phlinearcombination}}\\
    &\leq \abs{\sum_{j\in[n]}c_j(w)\left(\left\langle\hat\xib_h^k,\psib\left(s,a,w^{(j)}\right)\right\rangle- \left\langle\hat\thetab_{h}^{k(j)},\phib(s,a)\right\rangle\right)}\nn\\
    &+\abs{\sum_{j\in[n]}c_j(w)\left\langle\hat\thetab_h^{k(j)}-\tilde\thetab_{h}^k\left(w^{(j)}\right),\phib(s,a)\right\rangle}\nn+\abs{\sum_{j\in[n]}c_j(w)\left\langle\tilde\thetab_h^{k}\left(w^{(j)}\right)-\thetab_{h}^k\left(w^{(j)}\right),\phib(s,a)\right\rangle}\nn\\
    &=\abs{\sum_{j\in[n]}c_j(w)\left\langle\hat\thetab_h^{k(j)}-\tilde\thetab_{h}^k\left(w^{(j)}\right),\phib(s,a)\right\rangle}\nn
    +\abs{\sum_{j\in[n]}c_j(w)\left\langle\tilde\thetab_h^{k}\left(w^{(j)}\right)-\thetab_{h}^k\left(w^{(j)}\right),\phib(s,a)\right\rangle}\tag{Eqn. \eqref{eq:zeroforallsanda}}\\
    &\leq 2L\beta\norm{\phib(s,a)}_{\left(\Lambdab_h^k\right)^{-1}}\tag{Lemma \ref{lemm:bellmanupdate2}}.
\end{align}


\subsection{Proof of Lemma \ref{lemm:UCB2}}\label{sec:proofofUCB2}
We first note that conditioned on events $\{\Ec_1(w)\}_{w\in\widetilde\Wc}$ , for all $(s,a,w,h,k)\in\Sc\times\Ac\times\widetilde\Wc\times[H]\times[K]$, it holds that
\begin{align}
    &\abs{r_h(s,a,w)+\left\langle\hat\xib_h^k,\psib(s,a,w)\right\rangle-Q_h^\pi(s,a,w)-\mathbb{P}_h\left[V_{h+1}^k(.,w)-V_{h+1}^\pi(.,w)\right](s,a)}\nn\\
    &=\abs{r_h(s,a,w)+\left\langle\hat\xib_h^k,\psib(s,a,w)\right\rangle-r_h(s,a,w)-\mathbb{P}_h\left[V_{h+1}^k(.,w)\right](s,a)}\nn\\
    &=\abs{\left\langle\hat\xib_h^k,\psib(s,a,w)\right\rangle-\mathbb{P}_h\left[V_{h+1}^k(.,w)\right](s,a)}\nn\\
    &\leq 2L\beta\norm{\phib(s,a)}_{\left(\Lambdab_h^k\right)^{-1}},\tag{Lemma \ref{lemm:mainkeylemma}}
\end{align}
for any policy $\pi$.

Now, we prove the lemma by induction. The statement holds for $H$ because $Q_{H+1}^k(.,.,.)=Q_{H+1}^\ast(.,.,.)=0$ and thus conditioned events $\{\Ec_1(w)\}_{w\in\widetilde\Wc}$, defined in \eqref{eq:event1}, for all $(s,a,w,k)\in\Sc\times\Ac\times\widetilde\Wc\times[K]$, we have

\begin{align}
   \abs{r_H(s,a,w)+\left\langle\hat\xib_H^k,\psib(s,a,w)\right\rangle-Q_H^{\ast}(s,a,w)}\leq 2L\beta\norm{\phib(s,a)}_{\left(\Lambdab_H^k\right)^{-1}}.\nn
\end{align}
Therefore, conditioned on events $\{\Ec_1(w)\}_{w\in\widetilde\Wc}$, for all $(s,a,w,k)\in\Sc\times\Ac\times\widetilde\Wc\times[K]$, we have
\begin{align}
   Q^\ast_H(s,a,w) &\leq r_H(s,a,w)+\left\langle\hat\xib_H^k,\psib(s,a,w)\right\rangle+2L\beta\norm{\phib(s,a)}_{(\Lambdab_H^k)^{-1}} \nn\\
    &=\left\{r_H(s,a,w)+\left\langle\hat\xib_H^k,\psib(s,a,w)\right\rangle+2L\beta\norm{\phib(s,a)}_{(\Lambdab_H^k)^{-1}}\right\}^+\nn\\
    &= Q_H^k(s,a,w)\nn,
\end{align}
where the first equality follows from the fact that $Q^\ast_H(s,a,w)\geq 0$. Now, suppose the statement holds at time-step $h+1$ and consider time-step $h$. Conditioned on events $\{\Ec_1(w)\}_{w\in\widetilde\Wc}$, for all $(s,a,w,h,k)\in\Sc\times\Ac\times\widetilde\Wc\times[H]\times[K]$, we have 
\begin{align}
    0&\leq r_h(s,a,w)+\left\langle\hat\xib_h^k,\psib(s,a,w)\right\rangle-Q_h^{\ast}(s,a,w)-\mathbb{P}_h\left[V_{h+1}^k(.,w)-V_{h+1}^{\ast}(.,w)\right](s,a)+2L\beta\norm{\phib(s,a)}_{\left(\Lambdab_h^k\right)^{-1}}\nn\\
    &\leq r_h(s,a,w)+\left\langle\hat\xib_h^k,\psib(s,a,w)\right\rangle-Q_h^{\ast}(s,a,w)+2L\beta\norm{\phib(s,a)}_{\left(\Lambdab_h^k\right)^{-1}}.\tag{Induction assumption}
\end{align}
Therefore, conditioned on events $\{\Ec_1(w)\}_{w\in\widetilde\Wc}$, for all $(s,a,w,h,k)\in\Sc\times\Ac\times\widetilde\Wc\times[H]\times[K]$, we have
\begin{align}
    Q^\ast_h(s,a,w)&\leq r_h(s,a,w)+\left\langle\hat\xib_h^k,\psib(s,a,w)\right\rangle+2L\beta\norm{\phib(s,a)}_{\left(\Lambdab_h^k\right)^{-1}}\nn\\
    &=\left\{r_h(s,a,w)+\left\langle\hat\xib_h^k,\psib(s,a,w)\right\rangle+2L\beta\norm{\phib(s,a)}_{\left(\Lambdab_h^k\right)^{-1}}\right\}^+\nn\\
    &= Q_h^k(s,a,w)\nn,
\end{align}
where the first equality follows from the fact that $Q^\ast_h(s,a,w)\geq 0$. This completes the proof.


\subsection{Proof of Theorem \ref{thm:regretcs}}
\label{sec:proofofregretcs}
First, we bound the number of times Algorithm \ref{alg:UCBlvd} updates $\hat\xib_h^k$, i.e., number of planning calls. Let $P$ be the total number of updates and $k_p$ be the episode at which, the agent did replanning for the $p$-th time. Note that $\det\Lambdab_{h}^1 ={\la}^{d}$ and $\det\Lambdab_{h}^K\leq {\rm{trace}}(\Lambdab_{h}^K/{d})^{d} \leq \left(\la+\frac{K}{{d}}\right)^{d}$,
and consequently:
\begin{align}
    \frac{\det\Lambdab_{h}^K}{\det\Lambdab_{h}^1} &= \prod_{p=1}^{P} \frac{\det\Lambdab_{h}^{k_{p}}}{\det\Lambdab_{h}^{k_{p-1}}} \leq \left(1+\frac{K}{{d}\la}\right)^{d},\nn
\end{align}
and therefore

\begin{align}\label{eq:first}
    \prod_{h=1}^H\frac{\det\Lambdab_{h}^K}{\det\Lambdab_{h}^1} &=\prod_{h=1}^H \prod_{p=1}^{P} \frac{\det\Lambdab_{h}^{k_{p}}}{\det\Lambdab_{h}^{k_{p-1}}} \leq \left(1+\frac{K}{{d}\la}\right)^{{d}H}.
\end{align}

Since $1\leq \frac{\det\Lambdab_{h}^{k_p}}{\det\Lambdab_{h}^{k_{p-1}}}$ for all $p\in[P]$, we can deduce from \eqref{eq:first} that    
\begin{align}
    \exists h\in[H] \quad \text{such that} \quad e < \frac{\det \Lambdab_{h}^{k}}{\det \Lambdab_{h}^{\tilde k}}\nn
\end{align}
happens for at most ${d}H\log\left(1+\frac{K}{d\la}\right)$ number of episodes $k\in[K]$. This concludes that the number of planing calls in \UCBlvd is ${d}H\log\left(1+\frac{K}{d\la}\right)$.

Now, we prove the regret bound. Let $\delta_h^k=V_h^{\tilde k}(s_h^k,w^k) - V_h^{\pi^k}(s_h^k,w^k)$ and $\xi_{h+1}^k = \mathbb{E}\left[\delta_{h+1}^k\vert s_h^k,a_h^k\right]-\delta_{h+1}^k$. Conditioned on events $\{\Ec_1(w)\}_{w\in\widetilde\Wc}$, for all $(s,a,w,h,k)\in\Sc\times\Ac\times\widetilde\Wc\times[H]\times[K]$, we have
\begin{align}
    Q_h^{\tilde k}(s,a,w) - Q_h^{\pi^k}(s,a,w) &=r_h(s,a,w)+ \left\langle\hat\xib_h^{\tilde k},\psib(s,a,w)\right\rangle- Q_h^{\pi^k}(s,a,w)+2L\beta\norm{\phib(s,a)}_{(\Lambdab_h^{\tilde k})^{-1}}\nn\\
    &\leq \mathbb{P}_h\left[V_{h+1}^{\tilde k}(.,w)-V_{h+1}^{\pi^k}(.,w)\right](s,a)+4L\beta\norm{\phib(s,a)}_{(\Lambdab_h^{\tilde k})^{-1}}\label{eq:somethinginthemiddleQ2}.
\end{align}

Note that $\delta_h^k \leq Q_h^{\tilde k}(s_h^k,a_h^k,w^k) - Q_h^{\pi^k}(s_h^k,a_h^k,w^k)$. Thus, combining \eqref{eq:somethinginthemiddleQ2}, Lemma \ref{lemm:bellmanupdate2}, and a union bound over $\widetilde\Wc$, we conclude that for all $(h,k)\in[H]\times[K]$, with probability at least $1-\delta$, it holds that gives
\begin{align}
   \delta_h^k \leq  \xi_{h+1}^k+\delta_{h+1}^k+4L\beta\norm{\phib(s_h^k,a_h^k)}_{(\Lambdab_h^{\tilde k})^{-1}}\nn.
\end{align}

Note that for any positive semi-definite matrices $\A$, $\B$, and $\C$ such that $\A=\B+\C$, we have:
    \begin{align}\label{eq:dett}
    \det(\A)\geq \det(\B),~ \det(\A)\geq \det(\C),
    \end{align}
    and for any $\x \neq 0$ (\cite[Lemm.~12]{abbasi2011improved}):
    \begin{align}\label{eq:det}
        \frac{\norm{\x}^2_{\A}}{\norm{\x}^2_{\B}}\leq \frac{\det(\A)}{\det(\B)} \quad \text{and}\quad \frac{\norm{\x}^2_{\B^{-1}}}{\norm{\x}^2_{\A^{-1}}}\leq \frac{\det(\A)}{\det(\B)}.
    \end{align}
    
Now, we complete the regret analysis following similar steps as those of Theorem \ref{thm:regret}'s proof:
\begin{align}
    R_K &= \sum_{k=1}^K V_1^{\ast}(s_1^k,w^k)-V_1^{\pi^k}(s_1^k,w^k)\nn\\
    &\leq \sum_{k=1}^K V_1^{\tilde k}(s_1^k,w^k)-V_1^{\pi^k}(s_1^k,w^k) \tag{Lemma \ref{lemm:UCB2}}\\
    &=\sum_{k=1}^K \delta_1^k\nn\\
    &\leq \sum_{k=1}^K\sum_{h=1}^H\xi_{h}^k+4L\beta\sum_{k=1}^K\sum_{h=1}^H\norm{\phib(s_h^k,a_h^k)}_{\left(\Lambdab_h^{\tilde k}\right)^{-1}}\nn\\
    &\leq \sum_{k=1}^K\sum_{h=1}^H\xi_{h}^k+4L\beta\sum_{k=1}^K\sum_{h=1}^H\norm{\phib(s_h^k,a_h^k)}_{\left(\Lambdab_h^{ k}\right)^{-1}}\sqrt{\frac{\det \Lambdab_h^k}{\det \Lambdab_h^{\tilde k}}}\tag{Eqn. \eqref{eq:det}}\\
    &\leq  2H\sqrt{T\log({d}T/\delta)} +8HL\beta\sqrt{2{d}K\log(1+K/\la)}\nn\\
    &\leq \Otilde\left(L\sqrt{\la(d^3+md^2)H^3T}\right)\nn.
\end{align}


\section{\UCBlvd with Unknown Rewards}\label{sec:unknonwreward}
In order for our analysis to go through, we need a slightly different completeness assumption as below:

\begin{myassum}\label{assum:comp2} Given feature maps $\phib:\Sc\times\Ac\rightarrow \mathbb{R}^{d}$ and $\psib:\Sc\times\Ac\times\Wc\rightarrow \mathbb{R}^{\dpr}$, consider function class
\begin{align}
    \mathcal{F} = &\left\{f: f(s,w) =\min\left\{\max_{a\in\Ac}\left\{ \langle\nub,\psib(s,a,w)\rangle+\beta\norm{\phib(s,a)}_{\Lambdab^{-1}}+\tilde\beta\norm{\psib(s,a,w)}_{\tilde\Lambdab^{-1}}\right\}^+,H\right\}\right.\nn\\
    &\left.,\nub\in\mathbb{R}^{\dpr},\Lambdab\in\mathbf{S}^{d}_{++},\tilde\Lambdab\in\mathbf{S}^{\dpr}_{++},\beta,\tilde\beta\in\mathbb{R}\right\}.\nn
\end{align}
Then for any $f\in\Fc$, and $h\in[H]$, there exists a vector $\xib_h^{f}\in\mathbb{R}^{\dpr}$ with $\norm{\xib_h^f}\leq H\sqrt{\dpr}$ such that 
\begin{align}
    \mathbb{P}_h\left[f(.,w)\right](s,a)=\langle\xib_h^f,\psib(s,a,w)\rangle.\nn
\end{align}
\end{myassum}

\subsection{Overview}
\begin{algorithm}[t]
\DontPrintSemicolon
\KwInput{$\Ac$, $\la$, $\delta$, $H$, $K$, $\beta$, $\tilde\beta$}
$\bf Set:$ $Q_{H+1}^k(.,.,.)=0,~\forall k\in[K]$, $\tilde k = 1$\;
 \For{{\rm episodes} $k=1,\ldots,K$}
  {
  Observe the initial state $s_1^k$ and the task context $w^k$.\;
\If{$\exists h\in[H]$ such that $\frac{\det \Lambdab_h^k}{\det \Lambdab_h^{\tilde k}}>e$ {\bf or} $\frac{\det \tilde\Lambdab_h^k}{\det \tilde\Lambdab_h^{\tilde k}}>e$}{
$\tilde k = k$\;
\For{{\rm time-steps} $h=H,\ldots,1$}{
  Compute $\hat\xib_h^k$ as in \eqref{eq:hatxibandhatthetaunknown}.
      }
      }
      \For{{\rm time-steps} $h=1,\ldots,H$}{ 
      Compute $Q_h^{\tilde k}(s_h^k,a,w^k)$ for all $a\in\Ac$ as in \eqref{eq:Q3}.\;
      Play $a_h^k=\argmax_{a\in \Ac}Q_h^{\tilde k}(s_h^k,a,w^k)$ and observe $s_{h+1}^k$ and $r_h^k$.
      }
  }
 \caption{\UCBlvd with Unknown Rewards}
  \label{alg:UCBlvdunknown}
\end{algorithm}
Let $\psib_h^\tau = \psib(s_h^\tau,a_h^\tau,w^\tau)$. \UCBlvd with unknown rewards works with the following action-value functions:
\begin{align}
   Q_h^k(s,a,w) = \left\{\left\langle\tilde\etab_h^k+\hat\xib_h^k,\psib(s,a,w)\right\rangle+\beta\norm{\phib(s,a)}_{(\Lambdab_h^k)^{-1}}+\tilde\beta\norm{\psib(s,a,w)}_{(\tilde\Lambdab_h^k)^{-1}}\right\}^+,\label{eq:Q3}
\end{align}
where
\begin{align}
    \tilde\etab_h^k = \left(\tilde\Lambdab_h^k\right)^{-1}\sum_{\tau=1}^{k-1}\psib_h^\tau.r_h^\tau\quad \text{and}\quad\tilde\Lambdab_h^k = \la \Iden_{md}+\sum_{\tau=1}^{k-1}\psib_h^\tau{\psib_h^\tau}^\top,
\end{align}
and
\begin{align}
\hat\xib_h^k, \left\{\hat\thetab_h^{k(j)}\right\}_{j\in[n]} = &\argmin_{\xib,\left\{\thetab^{(j)}\right\}_{j\in[n]}}\sum_{j\in[n]}\sum_{(s,a)\in\Dc}\left(\left\langle\thetab^{(j)},\phib(s,a)\right\rangle-\left\langle\xib,\psib\left(s,a,w^{(j)}\right)\right\rangle\right)^2\label{eq:hatxibandhatthetaunknown}\\
&\text{s.t.}~\norm{\thetab^{(j)}-\tilde\thetab_h^k\left(w^{(j)}\right)}_{\Lambdab_h^k}\leq\beta,~\forall j\in[n]\quad\text{and}\quad\norm{\xib}_2\leq H\sqrt{md}\nn,
\end{align}
$\Dc=\left\{(s,a):\phib(s,a)~\text{are}~d~\text{linearly independent vectors}.\right\}$, and $\tilde\thetab_h^k(w)$ and $\Lambdab_h^k$ are defined in \eqref{eq:tildethetahk} and \eqref{eq:matrixLambda}, respectively.

We note that compared to \eqref{eq:Q2}, action-value function defined in \eqref{eq:Q3} involves an extra term $\left\langle\tilde\etab_h^k,\psib(s,a,w)\right\rangle+\tilde\beta\norm{\psib(s,a,w)}_{(\tilde\Lambdab_h^k)^{-1}}$. This term is in fact an upper bound on $r_h(s,a,w)$. Specifically, from Theorem 2 in \cite{abbasi2011improved}, we know that for $\tilde\beta=\sqrt{\la md}$, it holds that
\begin{align}\label{eq:abbasi}
    \norm{\etab_h-\tilde\etab_h^k}_{\tilde\Lambdab_h^k}\leq \tilde\beta,~\forall(h,k)\in[H]\times[K].
\end{align}

\begin{theorem}\label{thm:regretcsunknown}
Let $T=KH$. Under Assumptions \ref{assum:linearMDP}, \ref{assum:boundedness}, \ref{assum:mapping}, and \ref{assum:comp2}, the number of planning calls in Algorithm \ref{alg:UCBlvdunknown} is at most $dH\log\left(1+\frac{K}{d\la}\right)+mdH\log\left(1+\frac{K}{md\la}\right)$, and there exists an absolute constant $c>0$ such that for any fixed $\delta\in(0,0.5)$, if we set $\la=1$, $\beta=cH\left(md\right)\sqrt{\log( mdT/\delta)}$ and $\tilde\beta=\sqrt{md}$ in Algorithm \ref{alg:UCBlvdunknown}, then with probability at least $1-2\delta$, it holds that
\begin{align}
    R_K\leq2H\sqrt{T\log(dT/\delta)} +4H\sqrt{K}\left(L\beta\sqrt{2d\log(1+K/\la)}+\tilde\beta\sqrt{2md\log(1+K/\la)}\right)\leq \Otilde\left(L\sqrt{m^2d^3H^3T}\right).\nn
\end{align}
\end{theorem}

\subsection{Necessary Analysis for the Proof of Theorem \ref{thm:regretcsunknown}}

\begin{lemma}\label{lemm:event3lemma}
Let $c_\beta$ be a constant in the definition of $\beta$. Then, under Assumptions \ref{assum:linearMDP}, \ref{assum:boundedness}, \ref{assum:mapping}, and \ref{assum:comp2}, for a fixed $w$, there is an absolute constant $c_0$ independent of $c_\beta$, such that for all $(h,k)\in[H]\times[K]$, with probability at least $1-\delta$ it holds that
\begin{align}
     \norm{\sum_{\tau=1}^{k-1}\phib_h^\tau.\left( V_{h+1}^k(s_{h+1}^\tau,w)-\mathbb{P}_h[V^k_{h+1}(.,w)](s_h^\tau,a_h^\tau)\right)}_{\left(\Lambdab_h^k\right)^{-1}}\leq c_0mdH\sqrt{\log((c_\beta+1) mdT/\delta)}\nn,
\end{align}
where $c_0$ and $c_\beta$ are two independent absolute constants.
\end{lemma}
\begin{proof}


We note that $\norm{\tilde\etab_h^k+\hat\xib_h^k}_2\leq H\sqrt{md}+K/\la$ and $\norm{\left(\Lambdab_{h}^k\right)^{-1}}\leq\frac{1}{\la}$ and $\norm{\left(\tilde\Lambdab_{h}^k\right)^{-1}}\leq\frac{1}{\la}$. Thus, Lemmas \ref{lemm:lemmaD.4inJinetal} and \ref{lemm:coveringnumberQ3} together imply that for all $(h,k)\in[H]\times[K]$, with probability at least $1-\delta$ it holds that
\begin{align}
     &\norm{\sum_{\tau=1}^{k-1}\phib_h^\tau\left( V_{h+1}^k(s_{h+1}^\tau,w)-\mathbb{P}_h[V^k_{h+1}(.,w)](s_h^\tau,a_h^\tau)\right)}^2_{\left(\Lambdab_h^k\right)^{-1}}\nn\\
     &\leq4H^2\left(\frac{{d}}{2}\log\left(\frac{k+\la}{\la}\right)+md\log(1+8H\sqrt{md}/\epsilon)+d^2\log\left(\frac{1+32L^2\beta^2\sqrt{d}}{\la\epsilon^2}\right)\right.\nn\\
      &\left.+m^2d^2\log\left(\frac{1+8\tilde\beta^2\sqrt{md}}{\la\epsilon^2}\right)+\log\left(\frac{1}{\delta}\right)\right)+\frac{8k^2\epsilon^2}{\la}\nn.
\end{align}
If we let $\epsilon = \frac{dH}{k}$ and
$\beta = c_\beta(md) H\sqrt{\log(mdT/\delta)}$, then, there exists an absolute constant $C>0$ that is independent of $c_\beta$ such that
\begin{align}
   \norm{\sum_{\tau=1}^{k-1}\phib_h^\tau\left( V_{h+1}^k(s_{h+1}^\tau,w)-\mathbb{P}_h[V^k_{h+1}(.,w)](s_h^\tau,a_h^\tau)\right)}^2_{\left(\Lambdab_h^k\right)^{-1}}\leq  C(m^2d^2)H^2\log\left((c_\beta+1)mdT/\delta\right).\nn
\end{align}
\end{proof}

\begin{lemma}\label{lemm:bellmanupdate3}
Under Assumptions \ref{assum:linearMDP}, \ref{assum:boundedness}, \ref{assum:mapping}, and \ref{assum:comp2}, if we let $\beta=cmdH\sqrt{\la\log({md}T/\delta)}$ with an absolute constant $c>0$, then the event
\begin{align}\label{eq:event3}
    \Ec_3(w)\coloneqq&\left\{\norm{\thetab_h^k(w)-\tilde\thetab_h^k(w)}_{\Lambdab_h^k}\leq \beta,~\forall(h,k)\in[H]\times[K]\right\}.
\end{align}
holds with probability at least $1-\delta$ for a fixed $w$.
\end{lemma}
\begin{proof}
The proof follows the same steps as those of Lemma \ref{lemm:bellmanupdate2}, except that it uses Lemma \ref{lemm:event3lemma} instead of Lemma \ref{lemm:event1lemma} due to different structure of action-value functions $Q_h^k$ in this section.
\end{proof}

\begin{lemma}\label{lemm:mainkeylemmaunknown}
Let $\widetilde\Wc=\{w^\tau:\tau\in[K]\}\cup\{w^{(j)}:j\in[n]\}$. Conditioned on events $\{\Ec_3(w)\}_{w\in\widetilde\Wc}$ defined in \eqref{eq:event3}, for all $(s,a,w,h,k)\in\Sc\times\Ac\times\widetilde\Wc\times[H]\times[K]$, it holds that
\begin{align}
    \abs{\left\langle\hat\xib_h^k,\psib(s,a,w)\right\rangle-\mathbb{P}_h\left[V_{h+1}^k(.,w)\right](s,a)}\leq 2L\beta\norm{\phib(s,a)}_{\left(\Lambdab_h^k\right)^{-1}}.\nn
\end{align}
\end{lemma}
\begin{proof}
The proof follows the exact same steps as those of Lemma \ref{lemm:mainkeylemma}'s proof.
\end{proof}

\begin{lemma}\label{lemm:UCB3}
    Let $\widetilde\Wc=\{w^\tau:\tau\in[K]\}\cup\{w^{(j)}:j\in[n]\}$. Conditioned on events $\{\Ec_3(w)\}_{w\in\widetilde\Wc}$ defined in \eqref{eq:event3}, and with $Q_h^k$ computed as in \eqref{eq:Q3}, it holds that $Q_h^k(s,a,w)\geq Q_h^\ast(s,a,w)$ for all $(s,a,w,h,k)\in\Sc\times\Ac\times\widetilde\Wc\times[H]\times[K]$.
   \end{lemma}
\begin{proof}
We first note that conditioned on events $\{\Ec_3(w)\}_{w\in\widetilde\Wc}$, for all $(s,a,w,h,k)\in\Sc\times\Ac\times\widetilde\Wc\times[H]\times[K]$, it holds that
\begin{align}
    &\abs{\left\langle\tilde\etab_h^k+\hat\xib_h^k,\psib(s,a,w)\right\rangle-Q_h^\pi(s,a,w)-\mathbb{P}_h\left[V_{h+1}^k(.,w)-V_{h+1}^\pi(.,w)\right](s,a)}\nn\\
    &=\abs{\left\langle\tilde\etab_h^k+\hat\xib_h^k,\psib(s,a,w)\right\rangle-r_h(s,a,w)-\mathbb{P}_h\left[V_{h+1}^k(.,w)\right](s,a)}\nn\\
    &\leq\abs{\left\langle\hat\xib_h^k,\psib(s,a,w)\right\rangle-\mathbb{P}_h\left[V_{h+1}^k(.,w)\right](s,a)}+\tilde\beta\norm{\psib(s,a,w)}_{\left(\tilde\Lambdab_h^k\right)^{-1}}\tag{Eqn. \eqref{eq:abbasi}}\\
    &\leq 2L\beta\norm{\phib(s,a)}_{\left(\Lambdab_h^k\right)^{-1}}+\tilde\beta\norm{\psib(s,a,w)}_{\left(\tilde\Lambdab_h^k\right)^{-1}},\tag{Lemma \ref{lemm:mainkeylemmaunknown}}
\end{align}
for any policy $\pi$.

Now, we prove the lemma by induction. The statement holds for $H$ because $Q_{H+1}^k(.,.,.)=Q_{H+1}^\ast(.,.,.)=0$ and thus conditioned events $\{\Ec_3(w)\}_{w\in\widetilde\Wc}$, defined in \eqref{eq:event3}, for all $(s,a,w,k)\in\Sc\times\Ac\times\widetilde\Wc\times[K]$, we have

\begin{align}
    \abs{\left\langle\tilde\etab_H^k+\hat\xib_H^k,\psib(s,a,w)\right\rangle-Q_H^\ast(s,a,w)}\leq 2L\beta\norm{\phib(s,a)}_{\left(\Lambdab_H^k\right)^{-1}}+\tilde\beta\norm{\psib(s,a,w)}_{\left(\tilde\Lambdab_H^k\right)^{-1}}.
\end{align}
Therefore, conditioned on events $\{\Ec_3(w)\}_{w\in\widetilde\Wc}$, for all $(s,a,w,k)\in\Sc\times\Ac\times\widetilde\Wc\times[K]$, we have
\begin{align}
    Q^\ast_H(s,a,w)&\leq \left\langle\tilde\etab_H^k+\hat\xib_H^k,\psib(s,a,w)\right\rangle+2L\beta\norm{\phib(s,a)}_{\left(\Lambdab_H^k\right)^{-1}}+\tilde\beta\norm{\psib(s,a,w)}_{\left(\tilde\Lambdab_H^k\right)^{-1}}\nn\\
    &=\left\{\left\langle\tilde\etab_H^k+\hat\xib_H^k,\psib(s,a,w)\right\rangle+2L\beta\norm{\phib(s,a)}_{\left(\Lambdab_H^k\right)^{-1}}+\tilde\beta\norm{\psib(s,a,w)}_{\left(\tilde\Lambdab_H^k\right)^{-1}}\right\}^+\nn\\
    &= Q_H^k(s,a,w)\nn,
\end{align}
where the first equality follows from the fact that $Q^\ast_H(s,a,w)\geq 0$. Now, suppose the statement holds at time-step $h+1$ and consider time-step $h$. Conditioned on events $\{\Ec_3(w)\}_{w\in\widetilde\Wc}$, for all $(s,a,w,h,k)\in\Sc\times\Ac\times\widetilde\Wc\times[H]\times[K]$, we have 
\begin{align}
    0&\leq \left\langle\tilde\etab_h^k+\hat\xib_h^k,\psib(s,a,w)\right\rangle-Q_h^{\ast}(s,a,w)-\mathbb{P}_h\left[V_{h+1}^k(.,w)-V_{h+1}^{\ast}(.,w)\right](s,a)\nn\\
    &+2L\beta\norm{\phib(s,a)}_{\left(\Lambdab_h^k\right)^{-1}}+\tilde\beta\norm{\psib(s,a,w)}_{\left(\tilde\Lambdab_h^k\right)^{-1}}\nn\\
    &\leq \left\langle\tilde\etab_h^k+\hat\xib_h^k,\psib(s,a,w)\right\rangle-Q_h^{\ast}(s,a,w)+2L\beta\norm{\phib(s,a)}_{\left(\Lambdab_h^k\right)^{-1}}+\tilde\beta\norm{\psib(s,a,w)}_{\left(\tilde\Lambdab_h^k\right)^{-1}}.\tag{Induction assumption}
\end{align}
Therefore, conditioned on events $\{\Ec_3(w)\}_{w\in\widetilde\Wc}$, for all $(s,a,w,h,k)\in\Sc\times\Ac\times\widetilde\Wc\times[H]\times[K]$, we have
\begin{align}
    Q^\ast_h(s,a,w)&\leq \left\langle\tilde\etab_h^k+\hat\xib_h^k,\psib(s,a,w)\right\rangle+2L\beta\norm{\phib(s,a)}_{\left(\Lambdab_h^k\right)^{-1}}+\tilde\beta\norm{\psib(s,a,w)}_{\left(\tilde\Lambdab_h^k\right)^{-1}}\nn\\
    &=\left\{\left\langle\tilde\etab_h^k+\hat\xib_h^k,\psib(s,a,w)\right\rangle+2L\beta\norm{\phib(s,a)}_{\left(\Lambdab_h^k\right)^{-1}}+\tilde\beta\norm{\psib(s,a,w)}_{\left(\tilde\Lambdab_h^k\right)^{-1}}\right\}^+\nn\\
    &= Q_h^k(s,a,w)\nn,
\end{align}
where the first equality follows from the fact that $Q^\ast_h(s,a,w)\geq 0$. This completes the proof.
\end{proof}

\subsection{Proof of Theorem \ref{thm:regretcsunknown}}

First, we bound the number of times Algorithm \ref{alg:UCBlvdunknown} updates $\hat\xib_h^k$, i.e., number of planning calls. Let $P$ be the total number of policy updates and $k_p$ be the episode at, the agent did replanning for the $p$-th time. Note that $\det\Lambdab_{h}^1 ={\la}^d$ and $\det\Lambdab_{h}^K\leq {\rm{trace}}(\Lambdab_{h}^K/d)^d \leq \left(\la+\frac{K}{d}\right)^d$,
and consequently:
\begin{align}
    \frac{\det\Lambdab_{h}^K}{\det\Lambdab_{h}^1} &= \prod_{p=1}^{P} \frac{\det\Lambdab_{h}^{k_{p}}}{\det\Lambdab_{h}^{k_{p-1}}} \leq \left(1+\frac{K}{d\la}\right)^d,\nn
\end{align}
and therefore

\begin{align}\label{eq:first2}
    \prod_{h=1}^H\frac{\det\Lambdab_{h}^K}{\det\Lambdab_{h}^1} &=\prod_{h=1}^H \prod_{p=1}^{P} \frac{\det\Lambdab_{h}^{k_{p}}}{\det\Lambdab_{h}^{k_{p-1}}} \leq \left(1+\frac{K}{d\la}\right)^{dH}.
\end{align}

We similarly have
\begin{align}\label{eq:second}
    \prod_{h=1}^H\frac{\det\tilde\Lambdab_{h}^K}{\det\tilde\Lambdab_{h}^1} &=\prod_{h=1}^H \prod_{p=1}^{P} \frac{\det\tilde\Lambdab_{h}^{k_{p}}}{\det\tilde\Lambdab_{h}^{k_{p-1}}} \leq \left(1+\frac{K}{md\la}\right)^{mdH}.
\end{align}
Since $1\leq\frac{\det\Lambdab_{h}^{k_p}}{\det\Lambdab_{h}^{k_{p-1}}}$ for all $p\in[P]$, we can deduce from \eqref{eq:first2} and  \eqref{eq:second} that    
\begin{align}
    \exists h\in[H] \quad \text{such that} \quad e < \frac{\det \Lambdab_{h}^{k}}{\det \Lambdab_{h}^{\tilde k}}\quad \text{or}\quad e < \frac{\det \tilde\Lambdab_{h}^{k}}{\det \tilde\Lambdab_{h}^{\tilde k}}
\end{align}
happens for at most $dH\log\left(1+\frac{K}{d\la}\right)+mdH\log\left(1+\frac{K}{md\la}\right)$ number of episodes $k\in[K]$. This concludes that number of planning calls in Algorithm \ref{alg:UCBlvdunknown} is at most $dH\log\left(1+\frac{K}{d\la}\right)+mdH\log\left(1+\frac{K}{md\la}\right)$.

Now, we prove the regret bound. Let $\delta_h^k=V_h^{\tilde k}(s_h^k,w^k) - V_h^{\pi^k}(s_h^k,w^k)$ and $\xi_{h+1}^k = \mathbb{E}\left[\delta_{h+1}^k\vert s_h^k,a_h^k\right]-\delta_{h+1}^k$. Conditioned on events $\{\Ec_3(w)\}_{w\in\widetilde\Wc}$, for all $(s,a,w,h,k)\in\Sc\times\Ac\times\widetilde\Wc\times[H]\times[K]$, we have
\begin{align}
    Q_h^{\tilde k}(s,a,w) - Q_h^{\pi^k}(s,a,w) &=\left\langle\tilde\etab_h^{\tilde k}+\hat\xib_h^{\tilde k},\psib(s,a,w)\right\rangle- Q_h^{\pi^k}(s,a,w)+2L\beta\norm{\phib(s,a)}_{(\Lambdab_h^{\tilde k})^{-1}}+\tilde\beta\norm{\psib(s,a,w)}_{(\tilde\Lambdab_h^{\tilde k})^{-1}}\nn\\
    &\leq \mathbb{P}_h\left[V_{h+1}^{\tilde k}(.,w)-V_{h+1}^{\pi^k}(.,w)\right](s,a)+4L\beta\norm{\phib(s,a)}_{(\Lambdab_h^{\tilde k})^{-1}}+2\tilde\beta\norm{\psib(s,a,w)}_{(\tilde\Lambdab_h^{\tilde k})^{-1}}\label{eq:somethinginthemiddleQ3}.
\end{align}

Note that $\delta_h^k \leq Q_h^{\tilde k}(s_h^k,a_h^k,w^k) - Q_h^{\pi^k}(s_h^k,a_h^k,w^k)$.  Thus, combining \eqref{eq:somethinginthemiddleQ3}, Lemma \ref{lemm:bellmanupdate3}, and a union bound over $\widetilde\Wc$, we conclude that for all $(h,k)\in[H]\times[K]$, with probability at least $1-\delta$, it holds that gives
\begin{align}
   \delta_h^k \leq  \xi_{h+1}^k+\delta_{h+1}^k+4L\beta\norm{\phib(s_h^k,a_h^k)}_{(\Lambdab_h^{\tilde k})^{-1}}+2\tilde\beta\norm{\psib(s_h^k,a_h^k,w^k)}_{(\tilde\Lambdab_h^{\tilde k})^{-1}}\nn.
\end{align}

Now, we complete the regret analysis following similar steps as those of Theorem \ref{thm:regret}'s proof:
\begin{align}
    R_K &= \sum_{k=1}^K V_1^{\ast}(s_1^k,w^k)-V_1^{\pi^k}(s_1^k,w^k)\nn\\
    &\leq \sum_{k=1}^K V_1^{\tilde k}(s_1^k,w^k)-V_1^{\pi^k}(s_1^k,w^k) \tag{Lemma \ref{lemm:UCB3}}\\
    &=\sum_{k=1}^K \delta_1^k\nn\\
    &\leq \sum_{k=1}^K\sum_{h=1}^H\xi_{h}^k+4L\beta\sum_{k=1}^K\sum_{h=1}^H\norm{\phib(s_h^k,a_h^k)}_{\left(\Lambdab_h^{\tilde k}\right)^{-1}}+2\tilde\beta\sum_{k=1}^K\sum_{h=1}^H\norm{\psib(s_h^k,a_h^k,w^k)}_{\left(\tilde\Lambdab_h^{\tilde k}\right)^{-1}}\nn\\
    &\leq \sum_{k=1}^K\sum_{h=1}^H\xi_{h}^k+4L\beta\sum_{k=1}^K\sum_{h=1}^H\norm{\phib(s_h^k,a_h^k)}_{\left(\Lambdab_h^{ k}\right)^{-1}}\sqrt{\frac{\det \Lambdab_h^k}{\det \Lambdab_h^{\tilde k}}}+2\tilde\beta\sum_{k=1}^K\sum_{h=1}^H\norm{\psib(s_h^k,a_h^k,w^k)}_{\left(\tilde\Lambdab_h^{ k}\right)^{-1}}\sqrt{\frac{\det \tilde\Lambdab_h^k}{\det \tilde\Lambdab_h^{\tilde k}}}\tag{\eqref{eq:det}}\\
    &\leq  2H\sqrt{T\log(dT/\delta)} +4H\sqrt{K}\left(L\beta\sqrt{2d\log(1+K/\la)}+\tilde\beta\sqrt{2md\log(1+K/\la)}\right)\nn\\
    &\leq \Otilde\left(L\sqrt{\la m^2d^3H^3T}\right)\nn.
\end{align}


\section{Relaxation of Assumption \ref{assum:mapping}}\label{sec:relaxedmapping}

In this section, we replace Assumption \ref{assum:mapping} with the following assumption:

\begin{myassum}\label{assum:mapping2}
There is a known set $\{w^{(1)},w^{(2)},\ldots,w^{(n)}\}$ of $n\leq \dpr$ tasks such that $\psib(s,a,w)\in{\rm{Span}}\left(\left\{\psib(s,a,w^{(j)})\right\}_{j\in[n]}\right)$ for all $(s,a,w)\in\Sc\times\Ac\times\Wc$. This implies that for any $(s,a,w)\in\Sc\times\Ac\times\Wc$, there exist coefficients  $\{c_j(s,a,w)\}_{j\in[n]}$ such that
\begin{align}\label{eq:cprimecoefficient2}
   \psib(s,a,w) = \sum_{j\in[n]}c_j(s,a,w)\psib\left(s,a,w^{(j)}\right).
\end{align}
Moreover, $\sum_{j\in[n]}\abs{c_j(s,a,w)}\leq L$ for all $(s,a,w)\in \Sc\times\Ac\times\Wc$.
\end{myassum}

Define the concatenated mapping $\tilde\psib:\Sc\times\Ac\times\Wc\rightarrow \mathbb{R}^{d+\dpr}$ such that $\tilde\psib(s,a,w)=\left[\phib(s,a)^\top,\psib(s,a,w)^\top\right]^\top$. For any $w\in\Wc$, define $\Dc(w)=\left\{(s,a):\tilde\psib(s,a,w)~\text{are}~d+\dpr~\text{linearly independent vectors}.\right\}$. 
Given Assumption \ref{assum:mapping2}, we modify the planning step of \UCBlvd to the following:
\begin{align}
\hat\xib_h^k, \left\{\hat\thetab_h^{k(j)}\right\}_{j\in[n]} = &\argmin_{\xib,\left\{\thetab^{(j)}\right\}_{j\in[n]}}\sum_{j\in[n]}\sum_{(s,a)\in\Dc\left(w^{(j)}\right)}\left(\left\langle\thetab^{(j)},\phib(s,a)\right\rangle-\left\langle\xib,\psib\left(s,a,w^{(j)}\right)\right\rangle\right)^2\label{eq:hatxibandhattheta2}\\
&\text{s.t.}~\norm{\thetab^{(j)}-\tilde\thetab_h^k\left(w^{(j)}\right)}_{\Lambdab_h^k}\leq\beta,~\forall j\in[n]\quad\text{and}\quad\norm{\xib}_2\leq H\sqrt{\dpr}\nn.
\end{align}
The only change we make in Algorithm \ref{alg:UCBlvd} is in Line \ref{line:xihat}, in which $\hat\xib_h^k$ is now computed as defined in \eqref{eq:hatxibandhattheta2}. We present this modification in Algorithm \ref{alg:modifiedUCBlvd} for completeness.

\begin{theorem}\label{thm:modifiedregretcs}
Let $T=KH$. Under Assumptions \ref{assum:linearMDP}, \ref{assum:boundedness}, \ref{assum:comp}, and \ref{assum:mapping2}, the number or planning calls in Algorithm \ref{alg:modifiedUCBlvd} is at most ${d}H\log\left(1+\frac{K}{d\la}\right)$ and there exists an absolute constant $c>0$ such that for any fixed $\delta\in(0,0.5)$, if we set $\la=1$ and $\beta=cH\left(d+\sqrt{\dpr}\right)\sqrt{\la\log( d\dpr T/\delta)}$ in Algorithm \ref{alg:modifiedUCBlvd}, then with probability at least $1-2\delta$, it holds that
\begin{align}
    R_K\leq  2H\sqrt{T\log({d}T/\delta)} +8HL\beta\sqrt{2{d}K\log(K)}\leq \Otilde\left(L\sqrt{(d^3+d\dpr)H^3T}\right).
\end{align}
\end{theorem}

\begin{algorithm}[t]
\DontPrintSemicolon
\KwInput{$\Ac$, $\la$, $\delta$, $H$, $K$, $\beta$}
$\bf Set:$ $Q_{H+1}^k(.,.,.)=0,~\forall k\in[K]$, $\tilde k = 1$\;
 \For{{\rm episodes} $k=1,\ldots,K$}
  {
  Observe the initial state $s_1^k$ and the task context $w^k$.\;
\If{$\exists h\in[H]$ such that $\frac{\det \Lambdab_h^k}{\det \Lambdab_h^{\tilde k}}>e$}{
$\tilde k = k$\;
\For{{\rm time-steps} $h=H,\ldots,1$}{
  Compute $\hat\xib_h^k$ as in \eqref{eq:hatxibandhattheta2}.
      }
      }
      \For{{\rm time-steps} $h=1,\ldots,H$}{ 
      Compute $Q_h^{\tilde k}(s_h^k,a,w^k)$ for all $a\in\Ac$ as in \eqref{eq:Q2}.\;
      Play $a_h^k=\argmax_{a\in \Ac}Q_h^{\tilde k}(s_h^k,a,w^k)$ and observe $s_{h+1}^k$ and $r_h^k$.
      }
  }
 \caption{Modified \UCBlvd}
  \label{alg:modifiedUCBlvd}
\end{algorithm}

Proof of Theorem \ref{thm:modifiedregretcs}~follows exactly the same steps as those of Theorem \ref{thm:regretcs}. The only difference is the proof of Lemma \ref{lemm:mainkeylemma}, which we clarify in the proof of following lemma.

\begin{lemma}\label{lemm:mainkeylemma2}
Let $\widetilde\Wc=\{w^\tau:\tau\in[K]\}\cup\{w^{(j)}:j\in[n]\}$. Under Assumptions \ref{assum:linearMDP}, \ref{assum:boundedness}, \ref{assum:comp}, and \ref{assum:mapping2}, if we let $\beta=cH\left(d+\sqrt{\dpr}\right)\sqrt{\la\log(d\dpr T/\delta)}$ with an absolute constant $c>0$, then for all $(s,a,w,h,k)\in\Sc\times\Ac\times\Wc\times[H]\times[K]$ with probability at least $1-\delta$, it holds that
\begin{align}
    \abs{\left\langle\hat\xib_h^k,\psib(s,a,w)\right\rangle-\mathbb{P}_h\left[V_{h+1}^k(.,w)\right](s,a)}\leq 2L\beta\norm{\phib(s,a)}_{\left(\Lambdab_h^k\right)^{-1}}.\nn
\end{align}
\end{lemma}

\begin{proof}
We let $\tilde\psib_i(w)=\left[\phib_i^\top,\psib_i(w)^\top\right]^\top$ be the $i$-th element of $\tilde\Dc(w)=\left\{\tilde\psib(s,a,w):(s,a)\in\Dc(w)\right\}$ and for any triple $(s,a,w)\in\Sc\times\Ac\times\Wc$, we let $\{c^\prime_i(s,a,w)\}_{i\in[d+\dpr]}$ be the coefficients such that
\begin{align}
   \tilde\psib(s,a,w) = \sum_{i\in[d+\dpr]}c^\prime_i(s,a,w)\tilde\psib_i(w),\nn
\end{align}
which implies that 
\begin{align}
  \phib(s,a) = \sum_{i\in[d+\dpr]}c^\prime_i(s,a,w)\phib_i\quad\text{and}\quad\psib(s,a,w) = \sum_{i\in[d+\dpr]}c^\prime_i(s,a,w)\psib_i(w).\label{eq:coefficients2}
\end{align}

Thanks to Assumption \ref{assum:comp} and conditioned on events $\{\Ec_1(w)\}_{w\in\widetilde\Wc}$, one set of solution for \eqref{eq:hatxibandhattheta2} is $\left\{\thetab_h^k\left(w^{(j)}\right)\right\}_{j\in[n]}$ and $\xib_h^{V_{h+1}^k}$ with corresponding zero optimal objective value. Therefore, it holds that
\begin{align}
    \left\langle\hat\thetab_h^{k(j)},\phib_i\right\rangle=\left\langle\hat\xib_h^k,\psib_i\left(w^{(j)}\right)\right\rangle,\quad\forall (i,j)\in[d+\dpr]\times[n].\label{eq:implicationofoptimizationproblem2}
\end{align}

Moreover, for any triple $(s,a,j)\in\Sc\times\Ac\times[n]$, we have
\begin{align}
    \left\langle\hat\xib_h^k,\psib\left(s,a,w^{(j)}\right)\right\rangle &= \sum_{i\in[d+\dpr]}c^\prime_i\left(s,a,w^{(j)}\right)\left\langle\hat\xib_h^k,\psib_i\left(w^{(j)}\right)\right\rangle\tag{Eqn. \eqref{eq:coefficients2}}\\
    &= \sum_{i\in[d+\dpr]}c^\prime_i\left(s,a,w^{(j)}\right)\left\langle\hat\thetab_h^{k(j)},\phib_i\right\rangle\tag{Eqn. \eqref{eq:implicationofoptimizationproblem2}}\\
    &=\left\langle\hat\thetab_h^{k(j)},\phib(s,a)\right\rangle.\label{eq:zeroforallsanda2}
\end{align}

For any $(s,a,w)\in\Sc\times\Ac\times\Wc$, it holds that
\begin{align}
     \mathbb{P}_h\left[V_{h+1}^k(.,w)\right](s,a) &= \left\langle\thetab_{h}^k(w),\phib(s,a)\right\rangle \tag{Eqn. \eqref{eq:PVh+1linearform}}\\
     &= \left\langle\xib_h^{V_{h+1}^k},\psib(s,a,w)\right\rangle\tag{Assumption \ref{assum:comp}}\\
     &= \sum_{j\in[n]}c_j(s,a,w)\left\langle\xib_h^{V_{h+1}^k},\psib\left(s,a,w^{(j)}\right)\right\rangle\tag{Eqn. \eqref{eq:cprimecoefficient2}}\\
     &= \sum_{j\in[n]}c_j(s,a,w) \mathbb{P}_h\left[V_{h+1}^k\left(.,w^{(j)}\right)\right](s,a)\rangle\tag{Assumption \ref{assum:comp}}\\
     &= \sum_{j\in[n]}c_j(s,a,w) \left\langle\thetab_{h}^k\left(w^{(j)}\right),\phib(s,a)\right\rangle.\label{eq:Phlinearcombination2}
\end{align}

Finally, conditioned on events $\{\Ec_1(w)\}_{w\in\widetilde\Wc}$, for all $(s,a,w,h,k)\in\Sc\times\Ac\times\widetilde\Wc\times[H]\times[K]$, it holds that
\begin{align}
    &\abs{\left\langle\hat\xib_h^k,\psib(s,a,w)\right\rangle-\mathbb{P}_h\left[V_{h+1}^k(.,w)\right](s,a)}\\
    &=\abs{\left\langle\hat\xib_h^k,\psib(s,a,w)\right\rangle- \left\langle\thetab_{h}^k(w),\phib(s,a)\right\rangle}\nn\\
    &=\abs{\sum_{j\in[n]}c_j(s,a,w)\left(\left\langle\hat\xib_h^k,\psib\left(s,a,w^{(j)}\right)\right\rangle- \left\langle\thetab_{h}^k\left(w^{(j)}\right),\phib(s,a)\right\rangle\right)}\tag{Eqns. \eqref{eq:cprimecoefficient2} and \eqref{eq:Phlinearcombination}}\\
    &\leq \abs{\sum_{j\in[n]}c_j(s,a,w)\left(\left\langle\hat\xib_h^k,\psib\left(s,a,w^{(j)}\right)\right\rangle- \left\langle\hat\thetab_{h}^{k(j)},\phib(s,a)\right\rangle\right)}\nn\\
    &+\abs{\sum_{j\in[n]}c_j(s,a,w)\left\langle\hat\thetab_h^{k(j)}-\tilde\thetab_{h}^k\left(w^{(j)}\right),\phib(s,a)\right\rangle}\nn+\abs{\sum_{j\in[n]}c_j(s,a,w)\left\langle\tilde\thetab_h^{k}\left(w^{(j)}\right)-\thetab_{h}^k\left(w^{(j)}\right),\phib(s,a)\right\rangle}\nn\\
    &=\abs{\sum_{j\in[n]}c_j(s,a,w)\left\langle\hat\thetab_h^{k(j)}-\tilde\thetab_{h}^k\left(w^{(j)}\right),\phib(s,a)\right\rangle}\nn +\abs{\sum_{j\in[n]}c_j(s,a,w)\left\langle\tilde\thetab_h^{k}\left(w^{(j)}\right)-\thetab_{h}^k\left(w^{(j)}\right),\phib(s,a)\right\rangle}\tag{Eqn. \eqref{eq:zeroforallsanda}}\\
    &\leq 2L\beta\norm{\phib(s,a)}_{\left(\Lambdab_h^k\right)^{-1}}\tag{Lemma \ref{lemm:bellmanupdate2}}.
\end{align}
\end{proof}


\section{Standard \LLLSVI with Computation Sharing}\label{sec:standardLSVI}
In this section, we only rely on the following two assumptions:
\begin{myassum}\label{assum:comp3} Given a feature map $\psib:\Sc\times\Ac\times\Wc\rightarrow \mathbb{R}^{d^\prime}$, consider function class
\begin{align}
    \mathcal{F} = &\left\{f: f(s,w) =\min\left\{\max_{a\in\Ac}\left\{ \langle\nub,\psib(s,a,w)\rangle+\beta\norm{\psib(s,a,w)}_{\Lambdab^{-1}}\right\}^+,H\right\}\nub\in\mathbb{R}^{d^\prime},\beta\in\mathbb{R},\Lambdab\in\mathbf{S}^{d^\prime}_{++}\right\}.
\end{align}
Then for any $f\in\Fc$ and $h\in[H]$, there exists a vector $\nub^{f}_h\in\mathbb{R}^{d^\prime}$ with $\norm{\nub_h^f}_2\leq H\sqrt{\dpr}$ such that 
\begin{align}
    \mathbb{P}_h\left[f(.,w)\right](s,a)=\langle\psib(s,a,w),\nub^f_h\rangle.
\end{align}
Moreover, for every $h\in[H]$, there exists a vector $\etab_h$ such that $r_h(s,a,w)=\left\langle\etab_h,\psib(s,a,w)\right\rangle$.
\end{myassum}

\begin{myassum}\label{assum:boundedness2} Without loss of generality, $\norm{\psib(s,a,w)}_2\leq 1$ for all $(s,a,w)\in \Sc\times\Ac\times\Wc$, and $\norm{\etab_h}_2\leq \sqrt{\dpr}$ for all $h\in[H]$. 
\end{myassum}

\subsection{Overview}
Let $\psib_h^\tau = \psib(s_h^\tau,a_h^\tau,w^\tau)$. Standard \LLLSVI with computation sharing works with the following action-value functions:
\begin{align}
   Q_h^k(s,a,w) = \left\{r_h(s,a,w)+\left\langle\tilde\nub_h^k,\psib(s,a,w)\right\rangle+\beta\norm{\psib(s,a,w)}_{(\tilde\Lambdab_h^k)^{-1}}\right\}^+,\label{eq:Q5}
\end{align}
where
\begin{align}
   \tilde\nub_h^k = \left(\tilde\Lambdab_h^k\right)^{-1}\sum_{\tau=1}^{k-1}\psib_h^{\tau}.\min\left\{\max_{a\in\Ac}Q_{h+1}^k(s_{h+1}^\tau,a,w^\tau),H\right\}\quad\text{and}\quad\tilde\Lambdab_h^k = \la \Iden_{d^\prime}+\sum_{\tau=1}^{k-1}\psib_h^{\tau}{\psib_h^{\tau}}^\top.\label{eq:nub}
\end{align}

\begin{algorithm}[t]
\DontPrintSemicolon
\KwInput{$\Ac$, $\la$, $\delta$, $\beta$, $H$, $K$}
$\bf Set:$ $Q_{H+1}^k(.,.,.)=0,~\forall k\in[K]$, $\tilde k = 1$\;
 \For{{\rm episodes} $k=1,\ldots,K$}
  {
  Observe the initial state $s_1^k$ and the task context $w^k$.\;
\If{$\exists h\in[H]$ such that $\frac{\det \tilde\Lambdab_h^k}{\det \tilde\Lambdab_h^{\tilde k}}>e$}{
$\tilde k = k$\;
\For{{\rm time-steps} $h=H,\ldots,1$}{
  Compute Compute $\tilde\nub_h^{\tilde k}$ as in \eqref{eq:nub}.
      }
      }
      \For{{\rm time-steps} $h=1,\ldots,H$}{
      Compute $Q_h^{\tilde k}(s_h^k,a,w^k)$ for all $a\in\Ac$ as in \eqref{eq:Q5}.\;
      Play $a_h^k=\argmax_{a\in \Ac}Q_h^{\tilde k}(s_h^k,a,w^k)$ and observe $s_{h+1}^k$ and $r_h^k$.
      }
  }
 \caption{Standard \LLLSVI with Computation Sharing}
  \label{alg:standardLSVI}
\end{algorithm}

\begin{theorem}\label{thm:regretcs2}
Let $T=KH$. Under Assumptions \ref{assum:comp3} and \ref{assum:boundedness2}, the number of planning calls in \ref{alg:standardLSVI} is at most ${d^\prime}H\log\left(1+\frac{K}{\dpr\la}\right)$ and there exists an absolute constant $c>0$ such that for any fixed $\delta\in(0,0.5)$, if we set $\la=1$ and $\beta=c{d^\prime}H\sqrt{\log({d^\prime}T/\delta)}$ in Algorithm \ref{alg:standardLSVI}, then with probability at least $1-2\delta$, it holds that
\begin{align}
    R_K\leq  2H\sqrt{T\log({d^\prime}T/\delta)} +4H\beta\sqrt{2{d^\prime}K\log(K)}\leq \Otilde\left(\sqrt{{d^\prime}^3H^3T}\right).\nn
\end{align}
\end{theorem}

\subsection{Necessary Analysis for the Proof of Theorem \ref{thm:regretcs2}}
Thanks to Assumption \ref{assum:comp3}, we have

\begin{align}
    \mathbb{P}_h\left[V_{h+1}^k(.,w)\right](s,a) = \left\langle\nub_{h}^k,\psib(s,a,w)\right\rangle,\label{eq:linearPhforpsicomplet}
\end{align}
where $\nub_h^k=\nub_h^{V_{h+1}^k}$.

\begin{lemma}\label{lemm:event4lemma}
 Let $c_\beta$ be a constant in the definition of $\beta$. Then, under Assumption \ref{assum:boundedness2}, there is an absolute constant $c_0$ independent of $c_\beta$, such that for all $(h,k)\in[H]\times[K]$, with probability at least $1-\delta$ it holds that
\begin{align}
     \norm{\sum_{\tau=1}^{k-1}\psib_h^{\tau}.\left( V_{h+1}^k(s_{h+1}^\tau,w^\tau)-\mathbb{P}_h[V^k_{h+1}(.,w^\tau)](s_h^\tau,a_h^\tau)\right)}_{\left(\tilde\Lambdab_h^k\right)^{-1}}\leq c_0d^{\prime}H\sqrt{\log((c_\beta+1) d^{\prime}T/\delta)}\nn,
\end{align}
where $c_0$ and $c_\beta$ are two independent absolute constants.
\end{lemma}

\begin{proof}
We note that $\norm{\etab_h+\tilde\nub_h^k}_2\leq (1+H)\sqrt{\dpr}$ and $\norm{\left(\tilde\Lambdab_{h}^k\right)^{-1}}\leq\frac{1}{\la}$. Thus, Lemmas \ref{lemm:lemmaD.4inJinetal} and \ref{lemm:coveringnumberQ5} together imply that for all $(h,k)\in[H]\times[K]$, with probability at least $1-\delta$ it holds that
\begin{align}
     &\norm{\sum_{\tau=1}^{k-1}\phib_h^\tau\left( V_{h+1}^k(s_{h+1}^\tau,w^\tau)-\mathbb{P}_h[V^k_{h+1}(.,w^\tau)](s_h^\tau,a_h^\tau)\right)}^2_{\left(\tilde\Lambdab_h^k\right)^{-1}}\nn\\
     &\leq4H^2\left(\frac{{\dpr}}{2}\log\left(\frac{k+\la}{\la}\right)+\dpr\log(1+8H\sqrt{\dpr}/\epsilon)+\dpr^2\log\left(\frac{1+32L^2\beta^2\sqrt{\dpr}}{\la\epsilon^2}\right)+\log\left(\frac{1}{\delta}\right)\right)+\frac{8k^2\epsilon^2}{\la}\nn.
\end{align}
If we let $\epsilon = \frac{dH}{k}$ and
$\beta = c_\beta(\dpr+\sqrt{\dpr}) H\sqrt{\log(dT/\delta)}$, then, there exists an absolute constant $C>0$ that is independent of $c_\beta$ such that
\begin{align}
   \norm{\sum_{\tau=1}^{k-1}\phib_h^\tau\left( V_{h+1}^k(s_{h+1}^\tau,w^\tau)-\mathbb{P}_h[V^k_{h+1}(.,w^\tau)](s_h^\tau,a_h^\tau)\right)}^2_{\left(\tilde\Lambdab_h^k\right)^{-1}}\leq  C(\dpr+\dpr^2)H^2\log\left((c_\beta+1)\dpr T/\delta\right).\nn
\end{align}
\end{proof}

\begin{lemma}\label{lemm:bellmanupdate5}
Under Assumptions \ref{assum:comp3} and \ref{assum:boundedness2}, if we let $\beta=c{d^\prime}H\sqrt{\la\log({d^\prime}T/\delta)}$ with an absolute constant $c>0$, then the event
\begin{align}\label{eq:event4}
    \Ec_4\coloneqq&\left\{\norm{\nub_h^k-\tilde\nub_h^k}_{\tilde\Lambdab_h^k}\leq \beta,~\forall(h,k)\in[H]\times[K]\right\}.
\end{align}
holds with probability at least $1-\delta$.
\end{lemma}

\begin{proof}
\begin{align}
    \nub_h^k - \tilde\nub_h^k&= \nub_h^k-\left(\tilde\Lambdab_h^k\right)^{-1}\sum_{\tau=1}^{k-1}\psib_h^\tau V_{h+1}^k(s_{h+1}^\tau,w^\tau)\nn\\
    &= \left(\tilde\Lambdab_h^k\right)^{-1}\left(\tilde\Lambdab_h^k\nub_h^k-\sum_{\tau=1}^{k-1}\psib_h^{\tau} V_{h+1}^k(s_{h+1}^\tau,w^\tau)\right)\nn\\
    &=\underbrace{\la\left(\tilde\Lambdab_h^k\right)^{-1}\nub_h^k}_{\qb_1}\underbrace{-\left(\tilde\Lambdab_h^k\right)^{-1}\left(\sum_{\tau=1}^{k-1}\psib_h^{\tau}\left( V_{h+1}^k(s_{h+1}^\tau,w^\tau)-\mathbb{P}_h[V^k_{h+1}(.,w^\tau)](s_h^\tau,a_h^\tau)\right)\right)}_{\qb_2}\tag{Eqn. \eqref{eq:linearPhforpsicomplet}}.
\end{align}

Thus, in order to upper bound $\norm{\nub_h^k-\tilde\nub_h^k(w)}_{\tilde\Lambdab_h^k}$, we bound $\norm{\qb_1}_{\tilde\Lambdab_h^k}$ and $\norm{\qb_2}_{\tilde\Lambdab_h^k}$ separately.

From Assumption \ref{assum:boundedness2}, we have
\begin{align}\label{eq:q1phi3}
    \norm{\qb_1}_{\Lambdab_h^k}=\la\norm{\nub_h^k}_{\left(\tilde\Lambdab_h^k\right)^{-1}}\leq \sqrt{\la}\norm{\nub_h^k}_2\leq H\sqrt{\la \dpr}.
\end{align}

Thanks to Lemma \ref{lemm:event4lemma}, for all $(h,k)\in[H]\times[K]$, with probability at least $1-\delta$, it holds that
\begin{align} \label{eq:q2phi3}
    \norm{\qb_2}_{\tilde\Lambdab_h^k}\leq \norm{\sum_{\tau=1}^{k-1}\psib_h^{\tau}\left( V_{h+1}^k(s_{h+1}^\tau,w^\tau)-\mathbb{P}_h[V^k_{h+1}(.,w^\tau)](s_h^\tau,a_h^\tau)\right)}_{\left(\Lambdab_h^k\right)^{-1}}\leq c_0d^\prime H\sqrt{\log((c_\beta+1) d^\prime T/\delta)},
\end{align}
where $c_0$ and $c_\beta$ are two independent absolute constants.

Combining \eqref{eq:q1phi3} and \eqref{eq:q2phi3}, for all $(h,k)\in[H]\times[K]$, with probability at least $1-\delta$, it holds that

\begin{align}
    \norm{\nub_h^k-\tilde\nub_h^k}_{\tilde\Lambdab_h^k}\leq cd^\prime H\sqrt{\la\log(d^\prime T/\delta)}\nn
\end{align}
for some absolute constant $c>0$.
\end{proof}

\begin{lemma}\label{lemm:UCB5}
    Let the setting of Lemma  \ref{lemm:bellmanupdate5} holds. Conditioned on events $\Ec_4$ defined in \eqref{eq:event4}, and with $Q_h^k$ computed as in \eqref{eq:Q5}, it holds that $Q_h^k(s,a,w)\geq Q_h^\ast(s,a,w)$ for all $(s,a,w,h,k)\in\Sc\times\Ac\times\Wc\times[H]\times[K]$.
   \end{lemma}

\begin{proof}
We first note that conditioned on the event $\Ec_4$ , for all $(s,a,w,h,k)\in\Sc\times\Ac\times\Wc\times[H]\times[K]$, it holds that
\begin{align}
&\abs{r_h(s,a,w)+\left\langle\tilde\nub_h^k,\psib(s,a,w)\right\rangle-Q_h^\pi(s,a,w)-\mathbb{P}_h\left[V_{h+1}^k(.,w)-V_{h+1}^\pi(.,w)\right](s,a)}\nn\\
&=\abs{r_h(s,a,w)+\left\langle\tilde\nub_h^k,\psib(s,a,w)\right\rangle-r_h(s,a,w)-\mathbb{P}_h\left[V_{h+1}^k(.,w)\right](s,a)}\nn\\
&=\abs{\left\langle\tilde\nub_h^k,\psib(s,a,w)\right\rangle-\mathbb{P}_h\left[V_{h+1}^k(.,w)\right](s,a)}\nn\\
&=\abs{\left\langle\tilde\nub_h^k-\nub_h^k,\psib(s,a,w)\right\rangle}\nn\\
&\leq \norm{\tilde\nub_h^k-\nub_h^k}_{\tilde\Lambdab_h^k}\norm{\psib(s,a,w)}_{\left(\tilde\Lambdab_h^k\right)^{-1}}\nn\\
&\leq\beta\norm{\psib(s,a,w)}_{\left(\tilde\Lambdab_h^k\right)^{-1}},\tag{Lemma \ref{lemm:bellmanupdate5}}
\end{align}
for any policy $\pi$.

Now, we prove the lemma by induction. The statement holds for $H$ because $Q_{H+1}^k(.,.,.)=Q_{H+1}^\ast(.,.,.)=0$ and thus conditioned on the event $\Ec_4$, defined in \eqref{eq:event4}, for all $(s,a,w,k)\in\Sc\times\Ac\times\Wc\times[K]$, we have

\begin{align}
   \abs{r_h(s,a,w)+\left\langle\nub_H^k,\psib(s,a,w)\right\rangle-Q_H^{\ast}(s,a,w)}\leq \beta\norm{\psib(s,a,w)}_{\left(\tilde\Lambdab_H^k\right)^{-1}}.\nn
\end{align}
Therefore, conditioned on the event $\Ec_4$, for all $(s,a,w,k)\in\Sc\times\Ac\times\Wc\times[K]$, we have
\begin{align}
    Q^\ast_H(s,a,w) &\leq r_H(s,a,w)+\left\langle\nub_H^k,\psib(s,a,w)\right\rangle+\beta\norm{\psib(s,a,w)}_{(\tilde\Lambdab_H^k)^{-1}}\nn\\
    &= \left\{r_H(s,a,w)+\left\langle\nub_H^k,\psib(s,a,w)\right\rangle+\beta\norm{\psib(s,a,w)}_{(\tilde\Lambdab_H^k)^{-1}}\right\}^+\nn\\
    &=Q_H^k(s,a,w)\nn,
\end{align}
where the first equality follows from the fact that $Q^\ast_H(s,a,w)\geq 0$. Now, suppose the statement holds at time-step $h+1$ and consider time-step $h$. Conditioned on events $\Ec_4$, for all $(s,a,w,h,k)\in\Sc\times\Ac\times\Wc\times[H]\times[K]$, we have 
\begin{align}
    0&\leq r_h(s,a,w)+\left\langle\nub_h^k,\psib(s,a,w)\right\rangle-Q_h^{\ast}(s,a,w)-\mathbb{P}_h\left[V_{h+1}^k(.,w)-V_{h+1}^{\ast}(.,w)\right](s,a)+\beta\norm{\psib(s,a,w)}_{\left(\tilde\Lambdab_h^k\right)^{-1}}\nn\\
    &\leq r_h(s,a,w)+\left\langle\nub_h^k,\psib(s,a,w)\right\rangle-Q_h^{\ast}(s,a,w)+\beta\norm{\psib(s,a,w)}_{\left(\tilde\Lambdab_h^k\right)^{-1}}.\tag{Induction assumption}
\end{align}
Therefore, conditioned on events $\Ec_4$, for all $(s,a,w,h,k)\in\Sc\times\Ac\times\Wc\times[H]\times[K]$, we have
\begin{align}
    Q^\ast_h(s,a,w)&\leq r_h(s,a,w)+\left\langle\nub_h^k,\psib(s,a,w)\right\rangle+\beta\norm{\psib(s,a,w)}_{\left(\tilde\Lambdab_h^k\right)^{-1}}\nn\\
    &=\left\{r_h(s,a,w)+\left\langle\nub_h^k,\psib(s,a,w)\right\rangle+\beta\norm{\psib(s,a,w)}_{\left(\tilde\Lambdab_h^k\right)^{-1}}\right\}^+\nn\\
    &= Q_h^k(s,a,w)\nn,
\end{align}
where the first equality follows from the fact that $Q^\ast_H(s,a,w)\geq 0$. This completes the proof.

\end{proof}

\subsection{Proof of Theorem \ref{thm:regretcs2}}
First, we bound the number of times Algorithm \ref{alg:standardLSVI} updates $\tilde\nub_h^k$. Let $P$ be the total number of updates and $k_p$ be the episode at which, the agent did replanning for the $p$-th time. Note that $\det\tilde\Lambdab_{h}^1 ={\la}^{d^\prime}$ and $\det\tilde\Lambdab_{h}^K\leq {\rm{trace}}(\tilde\Lambdab_{h}^K/{d^\prime})^{d^\prime} \leq \left(\la+\frac{K}{{d^\prime}}\right)^{d^\prime}$,
and consequently:
\begin{align}
    \frac{\det\tilde\Lambdab_{h}^K}{\det\tilde\Lambdab_{h}^1} &= \prod_{p=1}^{P} \frac{\det\tilde\Lambdab_{h}^{k_{p}}}{\det\tilde\Lambdab_{h}^{k_{p-1}}} \leq \left(1+\frac{K}{{d^\prime}\la}\right)^{d^\prime},\nn
\end{align}
and therefore

\begin{align}\label{eq:first3}
    \prod_{h=1}^H\frac{\det\tilde\Lambdab_{h}^K}{\det\tilde\Lambdab_{h}^1} &=\prod_{h=1}^H \prod_{p=1}^{P} \frac{\det\tilde\Lambdab_{h}^{k_{p}}}{\det\tilde\Lambdab_{h}^{k_{p-1}}} \leq \left(1+\frac{K}{{d^\prime}\la}\right)^{{d^\prime}H}.
\end{align}

Since $1\leq \frac{\det\tilde\Lambdab_{h}^{k_p}}{\det\tilde\Lambdab_{h}^{k_{p-1}}}$ for all $p\in[P]$, we can deduce from \eqref{eq:first3} that    
\begin{align}
    \exists h\in[H] \quad \text{such that} \quad e < \frac{\det \tilde\Lambdab_{h}^{k}}{\det \tilde\Lambdab_{h}^{\tilde k}}\nn
\end{align}
happens for at most ${d^\prime}H\log\left(1+\frac{K}{\dpr\la}\right)$ number of episodes $k\in[K]$. This concludes that number of planning calls in Algorithm \ref{alg:standardLSVI} is at most ${d^\prime}H\log\left(1+\frac{K}{\dpr\la}\right)$.

Now, we prove the regret bound. Let $\delta_h^k=V_h^{\tilde k}(s_h^k,w^k) - V_h^{\pi^k}(s_h^k,w^k)$ and $\xi_{h+1}^k = \mathbb{E}\left[\delta_{h+1}^k\vert s_h^k,a_h^k\right]-\delta_{h+1}^k$. Conditioned on $\Ec_4$, for all $(s,a,w,h,k)\in\Sc\times\Ac\times\Wc\times[H]\times[K]$, we have
\begin{align}
    Q_h^{\tilde k}(s,a,w) - Q_h^{\pi^k}(s,a,w)&=r_h(s,a,w)+ \left\langle\thetab_h^{\tilde k},\psib(s,a,w)\right\rangle- Q_h^{\pi^k}(s,a,w)+\beta\norm{\psib(s,a,w)}_{(\tilde\Lambdab_h^{\tilde k})^{-1}}\nn\\
    &\leq \mathbb{P}_h\left[V_{h+1}^{\tilde k}(.,w)-V_{h+1}^{\pi^k}(.,w)\right](s,a)+2\beta\norm{\psib(s,a,w)}_{(\tilde\Lambdab_h^v)^{-1}}\label{eq:somethinginthemiddleQ5}.
\end{align}

Note that $\delta_h^{\tilde k} 
\leq Q_h^k(s_h^k,a_h^k,w^k) - Q_h^{\pi^k}(s_h^k,a_h^k,w^k)$. Thus, \eqref{eq:somethinginthemiddleQ5} and Lemma \ref{lemm:bellmanupdate5} imply that for all $(h,k)\in[H]\times[K]$, it holds that
\begin{align}
   \delta_h^k \leq  \xi_{h+1}^k+\delta_{h+1}^k+2\beta\norm{\psib(s_h^k,a_h^k,w^k)}_{(\tilde\Lambdab_h^k)^{-1}}\nn.
\end{align}

Now, we complete the regret analysis following similar steps as those of Theorem \ref{thm:regret}'s proof:
\begin{align}
    R_K &= \sum_{k=1}^K V_1^{\ast}(s_1^k,w^k)-V_1^{\pi^k}(s_1^k,w^k)\nn\\
    &\leq \sum_{k=1}^K V_1^{\tilde k}(s_1^k,w^k)-V_1^{\pi^k}(s_1^k,w^k) \tag{Lemma \ref{lemm:UCB5}}\\
    &=\sum_{k=1}^K \delta_1^k\nn\\
    &\leq \sum_{k=1}^K\sum_{h=1}^H\xi_{h}^k+2\beta\sum_{k=1}^K\sum_{h=1}^H\norm{\psib(s_h^k,a_h^k,w^k)}_{\left(\tilde\Lambdab_h^{\tilde k}\right)^{-1}}\nn\\
    &\leq \sum_{k=1}^K\sum_{h=1}^H\xi_{h}^k+2\beta\sum_{k=1}^K\sum_{h=1}^H\norm{\psib(s_h^k,a_h^k,w^k)}_{\left(\tilde\Lambdab_h^{ k}\right)^{-1}}\sqrt{\frac{\det \tilde\Lambdab_h^k}{\det \tilde\Lambdab_h^{\tilde k}}}\tag{Eqn. \eqref{eq:det}}\\
    &\leq  2H\sqrt{T\log({d^\prime}T/\delta)} +4H\beta\sqrt{2\la{d^\prime}K\log(1+K/\la)}\nn\\
    &\leq \Otilde\left(\sqrt{\la{d^\prime}^3H^3T}\right)\nn.
\end{align}


\section{Auxiliary Lemmas}\label{sec:auxiliary}
\paragraph{Notations.}

$\Nc_\epsilon(\Vc)$ denotes the $\epsilon$-covering number of the class $\Vc$ of functions mapping $
\Sc$ to $\mathbb{R}$ with respect to the distance ${\rm dist}(V,V^\prime)=\sup_s\abs{V(s)-V^\prime(s)}$.

\begin{lemma}[Bound on Weights $\thetab_h^k(w)$] \label{lemm:boundonweight} Under Assumption \ref{assum:linearMDP}, for any set of action-value functions $\{Q_h^k\}_{h\in[H]}$, and  $(w,h,k)\in\Wc\times[H]\times[K]$, it holds that
\begin{align}
    \norm{\thetab_h^k(w)}_2\leq H\sqrt{d}.\nn
\end{align}
\end{lemma}
\begin{proof}
Recall that $V_h^k(s,w)=\min\left\{\max_{a\in\Ac}Q_h^k(s,a,w),H\right\}$ and $\thetab_h^k(w)\coloneqq\int_{\Sc}V_{h+1}^k(s^\prime,w)d\mub_h(s^\prime)$. Thus, we have
\begin{align}
    \norm{\thetab_h^k(w)}_2=\norm{\int_{\Sc}V_{h+1}^k(s^\prime,w)d\mub_h(s^\prime)}\leq H\sqrt{d}.\nn
\end{align}
\end{proof}

\begin{lemma}[Lemma D.4 in \cite{jin2020provably}]\label{lemm:lemmaD.4inJinetal}
Let $\{s_\tau\}_{\tau=1}^\infty$ be a stochastic process on state space $\Sc$ with corresponding filtration $\{\Fc_\tau\}_{\tau=0}^\infty$. Let $\{\phib_\tau\}_{\tau=0}^\infty$ be an $\mathbb{R}^d$-valued stochastic process where $\phib_\tau\in\Fc_{\tau-1}$, and $\norm{\phib_\tau}\leq 1$. Let $\Lambdab_k=\la \Iden_{d}+\sum_{\tau=1}^{k-1}\phib_\tau\phib_\tau^\top$. Then with probability at least $1-\delta$, for all $k\geq 0$ and $V\in\Vc$ such that $\sup_{s\in\Sc}\abs{V(s)}\leq H$, we have
\begin{align}
    \norm{\sum_{\tau=1}^k\phib_\tau.\left(V(s_\tau)-\mathbb{E}\left[V(s_\tau)\vert\Fc_{\tau-1}\right]\right)}_{\Lambdab_k^{-1}}^2\leq 4H^2\left(\frac{{d}}{2}\log\left(\frac{k+\la}{\la}\right)+\log\left(\frac{\Nc_\epsilon(\Vc)}{\delta}\right)\right)+\frac{8k^2\epsilon^2}{\la}.\nn
\end{align}
\end{lemma}
\begin{lemma}\label{lemm:coveringnumber} For any $\epsilon>0$, the $\epsilon$-covering number of the Euclidean
ball in $\mathbb{R}^{d}$ with radius $R>0$ is upper bounded by $(1+2R/\epsilon)^{d}$.
\end{lemma}

\begin{lemma}\label{lemm:coveringnumberQ1}
For a fixed $w$, let $\Vc$ denote a class of functions mapping from $\Sc$ to $\mathbb{R}$ with following parametric form 
\begin{align}
     V(.) = \min\left\{\max_{a\in\Ac}\left\langle\z,\psib(.,a,w)\right\rangle+\left\langle\y,\phib(.,a)\right\rangle+\beta\sqrt{\phib(.,a)^\top \Yb\phib(.,a)},H\right\}\nn,
\end{align}
where the parameters $\beta\in\mathbb{R}$, $\z\in\mathbb{R}^{\dpr}$, $\y\in\mathbb{R}^{d}$, and $\Yb\in\mathbb{R}^{d\times d}$ satisfy $0\leq\beta\leq B$, $\norm{\z}\leq z$, $\norm{\y}\leq y$, and $\norm{\Yb}\leq \la^{-1}$. Assume $\norm{\phib(s,a)}\leq 1$ and $\norm{\psib(s,a,w)}\leq 1$ for all $(s,a,w)\in\Sc\times\Ac\times\Wc$. Then 
\begin{align}
    \log \left(\Nc_\epsilon(\Vc)\right)\leq \dpr\log(1+4z/\epsilon)+d\log(1+4y/\epsilon)+d^2\log\left(\frac{1+8B^2\sqrt{d}}{\la\epsilon^2}\right).\nn
\end{align}
\end{lemma}

\begin{proof}
First, we reparametrize $\Vc$ by letting $\tilde\Yb = \beta^2\Yb$. We have
\begin{align}
     V(.) = \min\left\{\max_{a\in\Ac}\left\langle\z,\psib(.,a,w)\right\rangle+\left\langle\y,\phib(.,a)\right\rangle+\sqrt{\phib(.,a)^\top \tilde\Yb\phib(.,a)},H\right\}\nn,
\end{align}
for $\norm{\z}\leq z$, $\norm{\y}\leq y$, and $\norm{\tilde\Yb}\leq \frac{B^2}{\la}$. For any two functions $V_1,V_2\in\Vc$ with parameters $\left(\z^1,\y^1,\tilde\Yb^1\right)$ and $\left(\z^2,\y^2,\tilde\Yb^2\right)$, respectively, we have
\begin{align}
  {\rm dist}(V_1,V_2)&\leq \sup_{(s,a)\in\Sc\times\Ac}\left|\left[\left\langle\z^1,\psib(s,a,w)\right\rangle+\left\langle\y^1,\phib(s,a)\right\rangle+\sqrt{\phib(s,a)^\top \tilde\Yb^1\phib(s,a)}\right]\right.\nn\\
  &\left.-\left[\left\langle\z^2,\psib(s,a,w)\right\rangle+\left\langle\y^2,\phib(s,a)\right\rangle+\sqrt{\phib(s,a)^\top \tilde\Yb^2\phib(s,a)}\right]\right|\nn\\
  &\leq\sup_{\psib:\norm{\psib}\leq 1, \phib:\norm{\phib}\leq 1}\abs{\left[\left\langle\z^1,\psib\right\rangle+\left\langle\y^1,\phib\right\rangle+\sqrt{\phib^\top \tilde\Yb^1 \phib}\right]-\left[\left\langle\z^2,\psib\right\rangle+\left\langle\y^2,\phib\right\rangle+\sqrt{\phib^\top \tilde\Yb^2 \phib}\right]}\nn\\
  &\leq\sup_{\psib:\norm{\psib}\leq 1}\abs{\left\langle\z^1-\z^2,\psib\right\rangle}+\sup_{\phib:\norm{\phib}\leq 1}\abs{\left\langle\y^1-\y^2,\phib\right\rangle}+\sup_{\phib:\norm{\phib}\leq 1}\sqrt{\abs{\phib^\top \left(\tilde\Yb^1-\tilde\Yb^2\right) \phib}}\tag{because $\abs{\sqrt{a}-\sqrt{b}}\leq\sqrt{\abs{a-b}}$ for $a,b\geq0$}\\
  &= \norm{\z^1-\z^2}+\norm{\y^1-\y^2}+\sqrt{\norm{\tilde\Yb^1-\tilde\Yb^2}}\nn\\
  &\leq\norm{\z^1-\z^2}+\norm{\y^1-\y^2}+\sqrt{\norm{\tilde\Yb^1-\tilde\Yb^2}_F}\label{eq:lastinthecoveringproofQ1}.
\end{align}
Let $\Cc_{\z}$ and $\Cc_{\y}$ be $\epsilon/2$-covers of $\{\z\in\mathbb{R}^{\dpr}:\norm{\z}\leq z\}$ and $\{\y\in\mathbb{R}^{d}:\norm{\y}\leq y\}$, respectively, with respect to the $2$-norm, and $\Cc_{\Yb}$ be an $\epsilon^2/4$-cover of $\{\Yb\in\mathbb{R}^{d\times d}:\norm{\Yb}_F\leq \frac{B^2\sqrt{d}}{\la}\}$, with respect to the Frobenius norm. By Lemma \ref{lemm:coveringnumber}, we know
\begin{align}
     \abs{\Cc_{\z}}\leq (1+4z/\epsilon)^{\dpr},\quad\abs{\Cc_{\y}}\leq (1+4y/\epsilon)^{d},\quad\abs{\Cc_{\Yb}}\leq \left(\frac{1+8B^2\sqrt{d}}{\la\epsilon^2}\right)^{d^2}.\nn
\end{align}
According to \eqref{eq:lastinthecoveringproofQ1}, it holds that $\Nc_\epsilon(\Vc)\leq\abs{\Cc_{\z}} \abs{\Cc_{\y}}\abs{\Cc_{\Yb}}$, and therefore
\begin{align}
    \log\left(\Nc_\epsilon(\Vc)\right)\leq \dpr\log(1+4z/\epsilon)+d\log(1+4y/\epsilon)+d^2\log\left(\frac{1+8B^2\sqrt{d}}{\la\epsilon^2}\right).\nn
\end{align}
\end{proof}

\begin{lemma}\label{lemm:coveringnumberQ2}
For a fixed $w$, let $\Vc$ denote a class of functions mapping from $\Sc$ to $\mathbb{R}$ with following parametric form 
\begin{align}
    V(.) = \min\left\{\max_{a\in\Ac}\left\{\left\langle\z,\psib(.,a,w)\right\rangle+2L\beta\sqrt{\phib(.,a)^\top \Yb\phib(.,a)}\right\}^+,H\right\}\nn,
\end{align}
where the parameters $\beta\in\mathbb{R}$, $\z\in\mathbb{R}^{\dpr}$ and $\Yb\in\mathbb{R}^{d\times d}$ satisfy $0\leq\beta\leq B$, $\norm{\z}\leq z$, and $\norm{\Yb}\leq \la^{-1}$. Assume $\norm{\phib(s,a)}\leq 1$ and $\norm{\psib(s,a,w)}\leq 1$ for all $(s,a,w)\in\Sc\times\Ac\times\Wc$. Then 
\begin{align}
    \log \left(\Nc_\epsilon(\Vc)\right)\leq \dpr\log(1+4z/\epsilon)+d^2\log\left(\frac{1+8B^2\sqrt{d}}{\la\epsilon^2}\right).\nn
\end{align}
\end{lemma}

\begin{proof}
First, we reparametrize $\Vc$ by letting $\tilde\Yb = \beta^2\Yb$. We have
\begin{align}
     V(.) = \min\left\{\max_{a\in\Ac}\left\langle\z,\psib(.,a,w)\right\rangle+\sqrt{\phib(.,a)^\top \tilde\Yb\phib(.,a)},H\right\}\nn,
\end{align}
for $\norm{\z}\leq z$, and $\norm{\tilde\Yb}\leq \frac{B^2}{\la}$. For any two functions $V_1,V_2\in\Vc$ with parameters $\left(\z^1,\tilde\Yb^1\right)$ and $\left(\z^2,\tilde\Yb^2\right)$, respectively, we have
\begin{align}
  {\rm dist}(V_1,V_2)&\leq \sup_{(s,a)\in\Sc\times\Ac}\abs{\left[\left\langle\z^1,\psib(s,a,w)\right\rangle+\sqrt{\phib(s,a)^\top \tilde\Yb^1\phib(s,a)}\right]-\left[\left\langle\z^2,\psib(s,a,w)\right\rangle+\sqrt{\phib(s,a)^\top \tilde\Yb^2\phib(s,a)}\right]}\nn\\
  &\leq\sup_{\psib:\norm{\psib}\leq 1, \phib:\norm{\phib}\leq 1}\abs{\left[\left\langle\z^1,\psib\right\rangle+\sqrt{\phib^\top \tilde\Yb^1 \phib}\right]-\left[\left\langle\z^2,\psib\right\rangle+\sqrt{\phib^\top \tilde\Yb^2 \phib}\right]}\nn\\
  &\leq\sup_{\psib:\norm{\psib}\leq 1}\abs{\left\langle\z^1-\z^2,\psib\right\rangle}+\sup_{\phib:\norm{\phib}\leq 1}\sqrt{\abs{\phib^\top \left(\tilde\Yb^1-\tilde\Yb^2\right) \phib}}\tag{because $\abs{\sqrt{a}-\sqrt{b}}\leq\sqrt{\abs{a-b}}$ for $a,b\geq0$}\\
  &= \norm{\z^1-\z^2}+\sqrt{\norm{\tilde\Yb^1-\tilde\Yb^2}}\nn\\
  &\leq\norm{\z^1-\z^2}+\sqrt{\norm{\tilde\Yb^1-\tilde\Yb^2}_F}\label{eq:lastinthecoveringproofQ2}.
\end{align}
Let $\Cc_{\z}$ be an $\epsilon/2$-cover of $\{\z\in\mathbb{R}^{\dpr}:\norm{\z}\leq z\}$ with respect to the $2$-norm, and $\Cc_{\Yb}$ be an $\epsilon^2/4$-cover of $\{\Yb\in\mathbb{R}^{d\times d}:\norm{\Yb}_F\leq \frac{B^2\sqrt{d}}{\la}\}$, with respect to the Frobenius norm. By Lemma \ref{lemm:coveringnumber}, we know
\begin{align}
    \abs{\Cc_{\z}}\leq (1+4z/\epsilon)^{\dpr},\quad\abs{\Cc_{\Yb}}\leq \left(\frac{1+8B^2\sqrt{d}}{\la\epsilon^2}\right)^{d^2}.\nn
\end{align}
According to \eqref{eq:lastinthecoveringproofQ2}, it holds that $\Nc_\epsilon(\Vc)\leq \abs{\Cc_{\z}}\abs{\Cc_{\Yb}}$, and therefore
\begin{align}
    \log\left(\Nc_\epsilon(\Vc)\right)\leq \dpr\log(1+4z/\epsilon)+d^2\log\left(\frac{1+8B^2\sqrt{d}}{\la\epsilon^2}\right).\nn
\end{align}
\end{proof}

\begin{lemma}\label{lemm:coveringnumberQ3}
For a fixed $w$, let $\Vc$ denote a class of functions mapping from $\Sc$ to $\mathbb{R}$ with following parametric form 
\begin{align}
     V(.) = \min\left\{\max_{a\in\Ac}\left\{\left\langle\z,\psib(.,a,w)\right\rangle+2L\beta\sqrt{\phib(.,a)^\top \Yb\phib(.,a)}+\tilde\beta\sqrt{\phib(.,a,w)^\top \tilde\Yb\phib(.,a,w)}\right\}^+,H\right\}\nn,
\end{align}
where the parameters $\beta,\tilde\beta\in\mathbb{R}$, $\z\in\mathbb{R}^{\dpr}$, $\Yb\in\mathbb{R}^{d\times d}$ and $\tilde\Yb\in\mathbb{R}^{\dpr\times \dpr}$ satisfy $0\leq\beta\leq B$, $0\leq\tilde\beta\leq \tilde B$ $\norm{\z}\leq z$, $\norm{\Yb}\leq \la^{-1}$ and $\norm{\tilde\Yb}\leq \la^{-1}$. Assume $\norm{\phib(s,a)}\leq 1$ and $\norm{\psib(s,a,w)}\leq 1$ for all $(s,a,w)\in\Sc\times\Ac\times\Wc$. Then 
\begin{align}
    \log \left(\Nc_\epsilon(\Vc)\right)\leq \dpr\log(1+4z/\epsilon)+d^2\log\left(\frac{1+8B^2\sqrt{d}}{\la\epsilon^2}\right)+\dpr^2\log\left(\frac{1+8\tilde B^2\sqrt{\dpr}}{\la\epsilon^2}\right).\nn
\end{align}
\end{lemma}

\begin{proof}
First, we reparametrize $\Vc$ by letting $\Zb = \beta^2\Yb$ and $\tilde\Zb = \tilde\beta^2\tilde\Yb$. We have
\begin{align}
     V(.) = \min\left\{\max_{a\in\Ac}\left\langle\z,\psib(.,a,w)\right\rangle+\sqrt{\phib(.,a)^\top \Zb\phib(.,a)}+\sqrt{\phib(.,a)^\top \tilde\Zb\phib(.,a)},H\right\}\nn,
\end{align}
for $\norm{\z}\leq z$, $\norm{\Zb}\leq \frac{B^2}{\la}$, and $\norm{\tilde\Zb}\leq \frac{\tilde B^2}{\la}$. For any two functions $V_1,V_2\in\Vc$ with parameters $\left(\z^1,\Zb^1, \tilde\Zb^1\right)$ and $\left(\z^2,\Zb^2, \tilde\Zb^2\right)$, respectively, we have
\begin{align}
  {\rm dist}(V_1,V_2)&\leq \sup_{(s,a)\in\Sc\times\Ac}\left|\left[\left\langle\z^1,\psib(s,a,w)\right\rangle+\sqrt{\phib(s,a)^\top \Zb^1\phib(s,a)}+\sqrt{\psib(s,a,w)^\top \tilde\Zb^1\psib(s,a,w)}\right]\right.\nn\\
  &\left.-\left[\left\langle\z^2,\psib(s,a,w)\right\rangle+\sqrt{\phib(s,a)^\top \Zb^2\phib(s,a)}+\sqrt{\psib(s,a,w)^\top \tilde\Zb^2\psib(s,a,w)}\right]\right|\nn\\
  &\leq\sup_{\psib:\norm{\psib}\leq 1, \phib:\norm{\phib}\leq 1}\abs{\left[\left\langle\z^1,\psib\right\rangle+\sqrt{\phib^\top \Zb^1 \phib}+\sqrt{\psib^\top\tilde\Zb^1 \psib}\right]-\left[\left\langle\z^2,\psib\right\rangle+\sqrt{\phib^\top\Zb^2 \phib}+\sqrt{\psib^\top \tilde\Zb^2 \psib}\right]}\nn\\
  &\leq\sup_{\psib:\norm{\psib}\leq 1}\abs{\left\langle\z^1-\z^2,\psib\right\rangle}+\sup_{\phib:\norm{\phib}\leq 1}\sqrt{\abs{\phib^\top \left(\Zb^1-\Zb^2\right) \phib}}+\sup_{\psib:\norm{\phib}\leq 1}\sqrt{\abs{\psib^\top \left(\tilde\Zb^1-\tilde\Zb^2\right) \psib}}\tag{because $\abs{\sqrt{a}-\sqrt{b}}\leq\sqrt{\abs{a-b}}$ for $a,b\geq0$}\\
  &= \norm{\z^1-\z^2}+\sqrt{\norm{\Zb^1-\Zb^2}}+\sqrt{\norm{\tilde\Zb^1-\tilde\Z^2}}\nn\\
  &\leq\norm{\z^1-\z^2}+\sqrt{\norm{\Zb^1-\Zb^2}_F}+\sqrt{\norm{\tilde\Zb^1-\tilde\Zb^2}_F}\label{eq:lastinthecoveringproofQ3}.
\end{align}
Let $\Cc_{\z}$ be an $\epsilon/2$-cover of $\{\z\in\mathbb{R}^{\dpr}:\norm{\z}\leq z\}$ with respect to the $2$-norm, $\Cc_{\Zb}$ be an $\epsilon^2/4$-cover of $\{\Zb\in\mathbb{R}^{d\times d}:\norm{\Zb}_F\leq \frac{B^2\sqrt{d}}{\la}\}$, and $\Cc_{\tilde\Zb}$ be an $\epsilon^2/4$-cover of $\{\tilde\Zb\in\mathbb{R}^{\dpr\times \dpr}:\norm{\tilde\Zb}_F\leq \frac{\tilde B^2\sqrt{d}}{\la}\}$ with respect to the Frobenius norm. By Lemma \ref{lemm:coveringnumber}, we know
\begin{align}
    \abs{\Cc_{\z}}\leq (1+4z/\epsilon)^{\dpr},\quad\abs{\Cc_{\Zb}}\leq \left(\frac{1+8B^2\sqrt{d}}{\la\epsilon^2}\right)^{d^2},\quad\abs{\Cc_{\tilde\Zb}}\leq \left(\frac{1+8\tilde B^2\sqrt{\dpr}}{\la\epsilon^2}\right)^{\dpr^2}.\nn
\end{align}
According to \eqref{eq:lastinthecoveringproofQ3}, it holds that $\Nc_\epsilon(\Vc)\leq \abs{\Cc_{\z}}\abs{\Cc_{\Yb}}$, and therefore
\begin{align}
    \log\left(\Nc_\epsilon(\Vc)\right)\leq \dpr\log(1+4z/\epsilon)+d^2\log\left(\frac{1+8B^2\sqrt{d}}{\la\epsilon^2}\right)+\dpr^2\log\left(\frac{1+8\tilde B^2\sqrt{\dpr}}{\la\epsilon^2}\right).\nn
\end{align}
\end{proof}

\begin{lemma}\label{lemm:coveringnumberQ5}
Let $\Vc$ denote a class of functions mapping from $\Sc$ to $\mathbb{R}$ with following parametric form 
\begin{align}
     V(.,.) = \min\left\{\max_{a\in\Ac}\left\{\left\langle\z,\psib(.,a,.)\right\rangle+2L\beta\sqrt{\psib(.,a,.)^\top \Yb\psib(.,a,.)}\right\}^+,H\right\}\nn,
\end{align}
where the parameters $\beta\in\mathbb{R}$, $\z\in\mathbb{R}^{\dpr}$ and $\Yb\in\mathbb{R}^{\dpr\times\dpr}$ satisfy $0\leq\beta\leq B$, $\norm{\z}\leq z$, and $\norm{\Yb}\leq \la^{-1}$. Assume $\norm{\psib(s,a,w)}\leq 1$ for all $(s,a,w)\in\Sc\times\Ac\times\Wc$. Then 
\begin{align}
    \log \left(\Nc_\epsilon(\Vc)\right)\leq \dpr\log(1+4z/\epsilon)+\dpr^2\log\left(\frac{1+8B^2\sqrt{\dpr}}{\la\epsilon^2}\right).\nn
\end{align}
\end{lemma}

\begin{proof}
First, we reparametrize $\Vc$ by letting $\tilde\Yb = \beta^2\Yb$. We have
\begin{align}
     V(.,.) = \min\left\{\max_{a\in\Ac}\left\langle\z,\psib(.,a,.)\right\rangle+\sqrt{\psib(.,a,.)^\top \tilde\Yb\psib(.,a,.)},H\right\}\nn,
\end{align}
for $\norm{\z}\leq z$, and $\norm{\tilde\Yb}\leq \frac{B^2}{\la}$. For any two functions $V_1,V_2\in\Vc$ with parameters $\left(\z^1,\tilde\Yb^1\right)$ and $\left(\z^2,\tilde\Yb^2\right)$, respectively, we have
\begin{align}
  {\rm dist}(V_1,V_2)&\leq \sup_{(s,a,w)\in\Sc\times\Ac\times\Wc}\left|\left[\left\langle\z^1,\psib(s,a,w)\right\rangle+\sqrt{\psib(s,a)^\top \tilde\Yb^1\psib(s,a)}\right]\right.\nn\\
  &\left.-\left[\left\langle\z^2,\psib(s,a,w)\right\rangle+\sqrt{\psib(s,a,w)^\top \tilde\Yb^2\psib(s,a,w)}\right]\right|\nn\\
  &\leq\sup_{\psib:\norm{\psib}\leq 1}\abs{\left[\left\langle\z^1,\psib\right\rangle+\sqrt{\psib^\top \tilde\Yb^1 \psib}\right]-\left[\left\langle\z^2,\psib\right\rangle+\sqrt{\psib^\top \tilde\Yb^2 \psib}\right]}\nn\\
  &\leq\sup_{\psib:\norm{\psib}\leq 1}\abs{\left\langle\z^1-\z^2,\psib\right\rangle}+\sup_{\psib:\norm{\psib}\leq 1}\sqrt{\abs{\psib^\top \left(\tilde\Yb^1-\tilde\Yb^2\right) \psib}}\tag{because $\abs{\sqrt{a}-\sqrt{b}}\leq\sqrt{\abs{a-b}}$ for $a,b\geq0$}\\
  &= \norm{\z^1-\z^2}+\sqrt{\norm{\tilde\Yb^1-\tilde\Yb^2}}\nn\\
  &\leq\norm{\z^1-\z^2}+\sqrt{\norm{\tilde\Yb^1-\tilde\Yb^2}_F}\label{eq:lastinthecoveringproofQ5}.
\end{align}
Let $\Cc_{\z}$ be an $\epsilon/2$-cover of $\{\z\in\mathbb{R}^{\dpr}:\norm{\z}\leq z\}$ with respect to the $2$-norm, and $\Cc_{\Yb}$ be an $\epsilon^2/4$-cover of $\{\Yb\in\mathbb{R}^{\dpr\times\dpr}:\norm{\Yb}_F\leq \frac{B^2\sqrt{\dpr}}{\la}\}$, with respect to the Frobenius norm. By Lemma \ref{lemm:coveringnumber}, we know
\begin{align}
    \abs{\Cc_{\z}}\leq (1+4z/\epsilon)^{\dpr},\quad\abs{\Cc_{\Yb}}\leq \left(\frac{1+8B^2\sqrt{\dpr}}{\la\epsilon^2}\right)^{\dpr^2}.\nn
\end{align}
According to \eqref{eq:lastinthecoveringproofQ5}, it holds that $\Nc_\epsilon(\Vc)\leq \abs{\Cc_{\z}}\abs{\Cc_{\Yb}}$, and therefore
\begin{align}
    \log\left(\Nc_\epsilon(\Vc)\right)\leq \dpr\log(1+4z/\epsilon)+\dpr^2\log\left(\frac{1+8B^2\sqrt{\dpr}}{\la\epsilon^2}\right).\nn
\end{align}
\end{proof}

\end{document}

%% file: commands.tex
\usepackage{xspace}

\newcommand{\LLLSVI}{Lifelong-LSVI\xspace}
\newcommand{\UCBlvd}{UCBlvd\xspace}

\newcommand{\la}{\lambda}

\newcommand{\nn}{\nonumber}

\newcommand{\bal}{\begin{align}}
\newcommand{\eal}{\end{align}}

\DeclarePairedDelimiterX{\inp}[2]{\langle}{\rangle}{#1, #2}

\newcommand{\phib}{\boldsymbol{\phi}}

\newcommand{\psib}{\boldsymbol{\psi}}

\newcommand{\qb}{\mathbf{q}}

\newcommand{\etab}{\boldsymbol{\eta}}

\newcommand{\C}{\mathbf{C}}
\newcommand{\A}{\mathbf{A}}
\newcommand{\B}{\mathbf{B}}
\newcommand{\Yb}{\mathbf{Y}}
\newcommand{\Vb}{\mathbf{V}}
\newcommand{\Z}{{Z}}

\newcommand{\x}{\mathbf{x}}

\newcommand{\y}{\mathbf{y}}

\newcommand{\z}{\mathbf{z}}

\newcommand{\Iden}{\mathbf{I}}

\newcommand{\rb}{\mathbf{r}}

\newcommand{\Zb}{\mathbf{Z}}

\newcommand{\Fc}{\mathcal{F}}

\newcommand{\Sc}{{\mathcal{S}}}

\newcommand{\Vc}{\mathcal{V}}
\newcommand{\Dc}{\mathcal{D}}

\newcommand{\Nc}{\mathcal{N}}

\newcommand{\Cc}{\mathcal{C}}

\newcommand{\Ac}{\mathcal{A}}

\newcommand{\Ec}{\mathcal{E}}

\newcommand{\Oc}{\mathcal{O}}

\newcommand{\Wc}{\mathcal{W}}

\newcommand{\beq}{\begin{equation}}
\newcommand{\eeq}{\end{equation}}
\newcommand{\bea}{\begin{align}}
\newcommand{\eea}{\end{align}}

\newcommand{\Otilde}{\tilde\Oc}

\newcommand{\Pb}{\mathbb{P}}

\newcommand{\mub}{\boldsymbol \mu}

\newcommand{\nub}{\boldsymbol \nu}

\newcommand{\thetab}{\boldsymbol\theta}

\newcommand{\rhob}{\boldsymbol\rho}
\newcommand{\xib}{\boldsymbol\xi}
\newcommand{\Lambdab}{\boldsymbol\Lambda}

\newcommand{\dpr}{{d^{\prime}}}